\documentclass{article}

\usepackage[main, final]{neurips_2026}

\usepackage{amsmath}
\usepackage{amssymb}
\usepackage{amsthm}
\usepackage{aliascnt}
\usepackage{mathtools}
\usepackage{nicefrac}
\usepackage{bm}

\usepackage{algorithm}
\usepackage{algpseudocode}

\usepackage{graphicx}
\usepackage{booktabs}
\usepackage{subcaption}

\usepackage{hyperref}
\usepackage{cleveref}

\usepackage{xcolor}
\usepackage{enumitem}
\usepackage{pifont}

\newcommand{\softmax}{\operatorname{softmax}}    
\newcommand{\onehot}{\operatorname{onehot}}      
\newcommand{\gumbel}{\operatorname{Gumbel}}      
\newcommand{\R}{\mathbb{R}}                      
\newcommand{\E}{\mathbb{E}}                      
\newcommand{\Var}{\operatorname{Var}}             
\newcommand{\Cov}{\operatorname{Cov}}             
\newcommand{\Tr}{\operatorname{Tr}}               
\newcommand{\norm}[1]{\left\| #1 \right\|}       

\newtheorem{theorem}{Theorem}
\newaliascnt{proposition}{theorem}
\newtheorem{proposition}[proposition]{Proposition}
\aliascntresetthe{proposition}
\newaliascnt{lemma}{theorem}

\aliascntresetthe{lemma}
\newaliascnt{corollary}{theorem}

\aliascntresetthe{corollary}
\newaliascnt{definition}{theorem}

\aliascntresetthe{definition}
\newaliascnt{remark}{theorem}
\newtheorem{remark}[remark]{Remark}
\aliascntresetthe{remark}

\crefname{proposition}{Proposition}{Propositions}
\crefname{lemma}{Lemma}{Lemmas}
\crefname{corollary}{Corollary}{Corollaries}
\crefname{definition}{Definition}{Definitions}
\crefname{remark}{Remark}{Remarks}

\title{On Reparameterizing the Score Function}

\author{
  Mohammad Hossein Amani$^{1}$ \quad
  Leonardo Bocchieri$^{1}$ \quad
  Maxime Peyrard$^{2}$ \\
  \textbf{Cl\'{e}ment Hongler}$^{1,3}$ \quad
  \textbf{Robert West}$^{1}$ \\[0.6em]
  $^{1}$EPFL \quad
  $^{2}$CNRS, Univ. Grenoble Alpes \quad
  $^{3}$XENT Labs
}

\begin{document}

\maketitle

\begin{abstract}
Policy gradients for autoregressive sequence models differentiate each token's log-probability while holding the conditioning history fixed. 
This discards information about how earlier decisions shape later conditional distributions, creating a credit-assignment bottleneck when rewards are sparse. 
We introduce History-Reparameterized (HR) policy gradients, which replace the frozen prefix in the score function with a differentiable reparameterized or relaxed history, allowing gradients of downstream log-probabilities to propagate backward through the generated trajectory. The resulting history-pathwise term is generally biased when added directly, so we derive Optimal HR (OHR), a centered control-variate estimator with a closed-form coefficient that minimizes the trace of the gradient variance at the population level. 
The variance reduction is governed by temporal structure in the policy dynamics: random or weakly structured policies provide little benefit, while policies with learned sequential dependencies yield informative history-pathwise corrections. For discrete tokens, the choice of relaxation is crucial: embedding-level relaxation yields substantial variance reduction that grows with horizon, whereas straight-through Gumbel-Softmax yields essentially none in our diagnostics. We validate the method on a Gaussian linear Markov model, a sparse-reward T-maze with a recurrent policy, and an autoregressive Transformer model for small Traveling Salesman Problem (TSP) instances.
\end{abstract}

\section{Introduction}
\label{sec:introduction}

Policy gradient methods for autoregressive sequence models---REINFORCE~\citep{williams1992},
PPO~\citep{schulman2017ppo}, or their descendants---are all built on the same primitive:
the score function $\nabla_\theta \log p_\theta(x_t \mid x_{<t})$.
Regardless of the algorithm built around it, this object treats the conditioning history $x_{<t}$ as a fixed constant,
differentiating only the \emph{direct} dependence on $\theta$ at step $t$ and discarding how $\theta$ shaped the
history that produced $x_{<t}$.
When a single scalar reward is assigned to an entire sequence, the resulting estimator is \textbf{order-agnostic}:
the causal dependence of $x_t$ on $x_{<t}$---the entire content of the autoregressive factorization---is invisible to the gradient,
and the variance of credit assignment to early tokens grows with horizon length~\citep{williams1992,schulman2015}.

To illustrate, consider a model solving a math problem step by step, with a binary reward for the correct final answer. Suppose the model makes an arithmetic error at an early step, then reasons correctly from that wrong intermediate result for the remainder. The score-function estimator assigns the same negative reward to every step. Yet the model's own conditional probabilities contain a diagnosis: the log-probabilities of the later steps, conditioned on the erroneous intermediate, are already high---the later reasoning is sound given its premise. If we could ask ``how would the downstream log-probabilities have changed had the early step been slightly different?'', we would obtain a gradient signal that is sensitive to the downstream consequences of the early step, rather than assigning the same scalar reward uniformly across all score terms.

In standard RL, this frozen-history treatment is not a design choice but a necessity.
The state transition is governed by an external environment whose dynamics are unknown and non-differentiable
with respect to the policy parameters---there is no gradient path from the policy to the states it visits.
Autoregressive sequence generation without an external environment is fundamentally different.
Each token's distribution is a function of the model's own previous outputs, and the conditioning history is itself
a differentiable function of $\theta$.
In this environment-less setting, treating the history as frozen is a convention inherited from the MDP framework,
not a structural constraint.

This raises a fundamental question: \emph{does the policy's own temporal structure contain information that can
improve credit assignment?}
If so, a policy with learned autoregressive structure is not merely a warm initialization for optimization---it is itself a
source of credit-assignment signal, with richer sequential dependencies yielding stronger signal.

We answer affirmatively by introducing \emph{History-Reparameterized} (HR) policy gradients, which replace the frozen history in the score function with a differentiable reparameterized or relaxed history. This exposes an additional history-pathwise term that propagates gradients of downstream log-probabilities backward through the generated prefix, crucially without requiring a differentiable reward. Because the reparameterization modifies only the conditioning argument of the log-probability, it applies to any algorithm built on the score function. The raw HR estimator is generally biased, so we derive Optimal HR (OHR), an unbiased control-variate estimator with a closed-form coefficient that minimizes the trace of the estimator variance (\Cref{thm:ohr_variance}). Because HR reparameterizes the \emph{conditioning history} rather than the reward, it exploits sequential structure that reward-reparameterization methods such as REBAR~\citep{tucker2017} and RELAX~\citep{grathwohl2018} cannot access; the two ideas are orthogonal and composable.

Experimentally, embedding-level history reparameterization yields variance reduction that grows with horizon (reaching approximately $5\times$ in the minimal embedding model), while straight-through Gumbel-Softmax yields negligible reduction, highlighting that the choice of relaxation is decisive. This effect appears in a minimal autoregressive model (\Cref{sec:exp_embedding}), a sparse-reward T-maze with a GRU policy (\Cref{sec:exp_tmaze}), and a small Transformer-style model for the TSP (\Cref{sec:exp_tsp}). The effect of variance reduction increases when HR is combined with reward-side methods such as RLOO or learned value functions \Cref{app:supp_composability}.

\section{Background}
\label{sec:background}

We study autoregressive sequence generation in the environment-less RL setting: the policy generates the entire trajectory, with no external dynamics.

\subsection{Autoregressive Policy and Objective}
\label{sec:bg_policy}

An autoregressive policy $p_\theta$ generates a sequence $x_{0:N}$ conditioned on a prompt $x_{<0}$:
\begin{equation}
    p_\theta(x_{0:N} \mid x_{<0}) = \prod_{t=0}^{N} p_\theta(x_t \mid x_{<t}),
    \label{eq:autoregressive}
\end{equation}
where each $x_t$ is either continuous ($x_t \in \R^d$) or discrete ($x_t \in \{e_1, \ldots, e_V\}$, the one-hot basis vectors over a vocabulary of size $V$). We write $\phi_\theta(x_{<t})$ for the output of the neural network (e.g., a Transformer) at step $t$. These are logits in the discrete case, or sufficient statistics (for instance, mean and variance for Gaussian random variables) in the continuous case so that $p_\theta(x_t \mid x_{<t}) = \mathcal{P}(\phi_\theta(x_{<t}))$ for some parametric family $\mathcal{P}$.

A reward function $r(x_{<0},\, x_{0:N}) \in \R$ scores the complete rollout given the prompt. The objective is
\begin{equation}
    J(\theta) = \E_{x_{0:N} \sim p_\theta(\cdot \mid x_{<0})}[r(x_{<0},\, x_{0:N})].
    \label{eq:objective}
\end{equation}
We do not assume $r$ is differentiable.

\subsection{Score-Function Estimator}
\label{sec:bg_score}

The score-function (REINFORCE) estimator~\citep{williams1992} gives an unbiased gradient:
\begin{equation}
    \nabla_\theta J = \E\Bigl[(r - b)\sum_{t=0}^{N} \nabla_\theta \log p_\theta(x_t \mid x_{<t})\Bigr],
    \label{eq:reinforce}
\end{equation}
where $b$ is any baseline independent of the sampled actions~\citep{williams1992,ranganath2014bbvi}. Each score term treats the conditioning history $x_{<t}$ as a fixed constant. When a single reward scores the entire sequence, every score term receives the same weight $(r - b)$, and the causal influence of early tokens on later log-probabilities is invisible to the gradient.

\subsection{Reparameterized Sampling}
\label{sec:bg_reparam}

For a large class of distributions, sampling can be written as a deterministic, differentiable transformation of parameter-free noise~\citep{kingma2014,rezende2014stochastic}. If there exists a noise distribution $p(\varepsilon)$ independent of $\theta$ and a differentiable function $S$ such that
\begin{equation}
    x_t = S(\phi_\theta(x_{<t}),\, \varepsilon_t), \qquad \varepsilon_t \sim p(\varepsilon),
    \label{eq:reparam_general}
\end{equation}
then $x_t$ retains a differentiable dependence on $\theta$ through $\phi_\theta$. Many continuous families admit exact reparameterization; see \Cref{app:extended_related_work} for a survey. For discrete variables, continuous relaxations provide differentiable approximations (\Cref{sec:bg_discrete}).

\paragraph{Reparameterized trajectory.}
When all steps admit reparameterization, we can unroll the entire trajectory as a differentiable function of noise. Define the \emph{reparameterized history} $\chi_{0:N}(\theta)$ recursively:
\begin{equation}
    \chi_t(\theta) = S(\phi_\theta(\chi_{<t}(\theta)),\, \varepsilon_t), \qquad \varepsilon_{0:N} \sim p(\varepsilon),
    \label{eq:chi_recursive}
\end{equation}
with $\chi_t(\theta) = x_t$ in value. By construction, $\chi_t$ depends on $\theta$ both directly through $\phi_\theta$ and indirectly through all earlier reparameterized variables $\chi_{<t}(\theta)$.

\paragraph{Pathwise estimator.}
When both the sampler and reward are differentiable, the reparameterization trick yields the pathwise gradient~\citep{kingma2014,rezende2014stochastic}:
\begin{equation}
    \nabla_\theta J = \E_{\varepsilon_{0:N}}\Bigl[\nabla_\theta\, r\bigl(\chi_{0:N}(\theta)\bigr)\Bigr],
    \label{eq:pathwise}
\end{equation}
which propagates credit recursively through the trajectory via BPTT. Each $\nabla_\theta \chi_t$ depends on $\theta$ directly and through all earlier $\chi_{<t}$, yielding the recursion:
\begin{equation}
    \nabla_\theta \chi_t = \frac{\partial S}{\partial \phi}\Bigl[\frac{\partial \phi_\theta}{\partial \theta} + \sum_{j < t} \frac{\partial \phi_\theta}{\partial \chi_j}\nabla_\theta \chi_j\Bigr].
    \label{eq:bptt_recursion}
\end{equation}
This estimator typically has much lower variance than the score-function estimator for smooth rewards and samplers, because it uses local derivatives rather than a global scalar signal. However, it requires a differentiable reward, which is unavailable in most RL settings (e.g., binary correctness).

\paragraph{The gap.}
The score-function estimator is unbiased and does not require a differentiable reward, but discards the temporal structure of the trajectory. The pathwise estimator exploits this structure but requires differentiating through $r$. We aim to combine both: retain the score function for the current token (so that $r$ need not be differentiable) while reparameterizing the \emph{conditioning history} to inject pathwise credit-assignment information.

\subsection{Reparameterization Strategies for Discrete Tokens}
\label{sec:bg_discrete}

For continuous actions, reparameterization is exact (e.g., Gaussian location-scale). For discrete tokens, we consider two relaxation strategies.

\paragraph{Gumbel-Softmax / Straight-Through (ST).}
Sampling from a categorical distribution can be written using the Gumbel-Softmax trick~\citep{jang2017,maddison2017}: $x_t = \onehot(\arg\max(\log p_\theta(x_t \mid x_{<t}) + g_t))$ where $g_t \sim \gumbel(0,1)$. The $\arg\max$ is non-differentiable, but the Gumbel-Softmax relaxation replaces it with $\softmax((\text{logits} + g_t)/\tau)$, yielding a soft approximation. In the straight-through variant, the forward pass uses the hard one-hot sample while the backward pass uses soft gradients.

\paragraph{Embedding-level reparameterization.}
In autoregressive models, each token enters subsequent steps as a continuous embedding vector $e_{x_t} \in \R^{d_e}$ via lookup in an embedding matrix $E \in \R^{V \times d_e}$. Instead of passing a hard one-hot vector, we can compute a soft weighted average of embeddings:
\begin{equation}
    \chi_t = \softmax\bigl(\phi_\theta(\chi_{<t}) / \tau + g_t\bigr) \cdot E,
    \label{eq:bg_embedding_reparam}
\end{equation}
which produces a continuous vector in $\R^{d_e}$ that depends differentiably on $\theta$. Note that standard embedding lookup is already a matrix multiply $\onehot(x_t) \cdot E$; the relaxation simply replaces the hard one-hot with a soft weight vector, requiring no architectural changes. The distinction from ST-Gumbel is that the relaxation creates a \emph{genuinely different forward-pass input} to the next step (a smooth blend of embedding vectors), rather than using the same hard one-hot in the forward pass and soft gradients only in the backward pass.

\section{History-Reparameterized Policy Gradients}
\label{sec:method}

We now present our main contribution: a family of gradient estimators that reparameterize the conditioning history in the policy gradient, enabling pathwise credit assignment without requiring a differentiable reward.

\subsection{The History-Reparameterized Estimator}
\label{sec:hr_estimator}

Starting from the score-function estimator (\Cref{eq:reinforce}), we replace the frozen history $x_{<t}$ with its differentiable reparameterization $\chi_{<t}(\theta)$ from \Cref{eq:chi_recursive}:
\begin{equation}
    g_{\mathrm{HR}} = (r - b)\sum_{t=0}^{N} \nabla_\theta \log p_\theta\bigl(x_t \mid \chi_{<t}(\theta)\bigr),
    \label{eq:g_hr}
\end{equation}
with stop-gradient on $r$ and $b$. For exact continuous reparameterizations, $\chi_{<t}(\theta) = x_{<t}$ in value while retaining a computation graph through $\theta$; for discrete relaxations, $\chi_{<t}$ is instead a differentiable surrogate history (see \Cref{app:discrete_relaxation_caveat}). In either case, the chain rule decomposes \Cref{eq:g_hr} into two terms. We write $\partial_\theta$ for the partial derivative holding $\chi_{<t}$ fixed and $\nabla_\theta$ for the total derivative:
\begin{equation}
    g_{\mathrm{HR}} = (r - b)\sum_{t=0}^{N} \biggl[\underbrace{\partial_\theta \log p_\theta(x_t \mid \chi_{<t})}_{\text{score term } S_t} + \underbrace{\sum_{j < t} \bigl(\nabla_\theta \chi_j\bigr)^\top \frac{\partial \log p_\theta(x_t \mid \chi_{<t})}{\partial \chi_j}}_{\text{history-pathwise term } C_t}\biggr],
    \label{eq:g_hr_decomp}
\end{equation}
where we write $\chi_{<t}$ for $\chi_{<t}(\theta)$. In the exact continuous case, $\chi_j = x_j$ and the score term reduces to standard REINFORCE; in the relaxed discrete case, evaluation is at the soft history. Defining
\begin{equation}
    S = \sum_{t=0}^{N} S_t, \qquad C = \sum_{t=0}^{N} C_t, \qquad H = (r - b)\,C,
    \label{eq:SC_def}
\end{equation}
we can write compactly:
\begin{equation}
    g_{\mathrm{HR}} = \underbrace{(r-b)\,S}_{g_{\mathrm{PG}}^b} + \underbrace{(r-b)\,C}_{H} = g_{\mathrm{PG}}^b + H.
    \label{eq:g_hr_compact}
\end{equation}
The first term is the standard baselined policy gradient; the second injects history-pathwise credit assignment via BPTT through the reparameterized trajectory. Backpropagating through a model's own relaxed, autoregressively generated discrete sequence was previously used to train discrete-bottleneck sequence autoencoders~\citep{amani2024sigmaae}; here we use such history gradients inside the policy-gradient score function.

\subsection{Bias Characterization}
\label{sec:bias}

The history-pathwise term $H$ generally has nonzero expectation, making $g_{\mathrm{HR}}$ biased. The following proposition characterizes the bias exactly.

\begin{proposition}[Bias of the HR estimator]
\label{prop:bias}
Let $g_{\mathrm{PG}}^b$ be the baselined REINFORCE estimator with $\E[g_{\mathrm{PG}}^b] = \nabla_\theta J$. Then
\begin{equation}
    \E[g_{\mathrm{HR}}] = \nabla_\theta J + \mathrm{Bias}(\theta),
    \label{eq:bias_identity}
\end{equation}
where
\begin{equation}
    \mathrm{Bias}(\theta) = \sum_{t=0}^{N} \E\Bigl[(r - b)\bigl(\nabla_\theta \log p_\theta(x_t \mid \chi_{<t}(\theta)) - \nabla_\theta \log p_\theta(x_t \mid x_{<t})\bigr)\Bigr].
    \label{eq:bias_term}
\end{equation}
\end{proposition}

\begin{proof}[Proof sketch]
By linearity and the unbiasedness of $g_{\mathrm{PG}}^b$:
\begin{align*}
    \E[g_{\mathrm{HR}}] &= \E[g_{\mathrm{PG}}^b] + \E[H] = \nabla_\theta J + \E[(r-b)\,C].
\end{align*}
Adding and subtracting the frozen-history score inside $g_{\mathrm{HR}}$ yields \Cref{eq:bias_term}. The full proof is given in \Cref{app:bias_proof}.
\end{proof}

For exact continuous reparameterizations, the sampled and reparameterized histories---and hence the evaluated log-probabilities---coincide in value. Their total derivatives nevertheless differ by the chain-rule terms through the history, whose reward-weighted expectation need not vanish, as illustrated by the Gaussian example in \Cref{sec:theory}. For discrete relaxations, the bias also includes the discrepancy from evaluating the log-probability at a surrogate history rather than the sampled history (see \Cref{app:discrete_relaxation_caveat}).

\subsection{Variance-Optimal Control Variate (OHR)}
\label{sec:ohr}

Since $H$ is biased, we cannot simply add it to $g_{\mathrm{PG}}^b$ and retain unbiasedness. Instead, we center $H$ to form a zero-mean control variate:
\begin{equation}
    \bar{H} = H - \E[H],
    \label{eq:H_bar}
\end{equation}
and construct the \emph{Optimal History-Reparameterized} (OHR) estimator:
\begin{equation}
    g_{\mathrm{OHR}} = g_{\mathrm{PG}}^b + \alpha_1\,\bar{H}.
    \label{eq:g_ohr}
\end{equation}
With exact population centering, $\E[\bar{H}] = 0$, so $g_{\mathrm{OHR}}$ is unbiased for \emph{any fixed} scalar $\alpha_1 \in \R$.

\begin{theorem}[Optimal OHR coefficient and variance reduction]
\label{thm:ohr_variance}
The scalar coefficient minimizing $\Tr(\Var(g_{\mathrm{OHR}}))$ is
\begin{equation}
    \alpha_1^* = -\frac{\Tr\bigl(\Cov(g_{\mathrm{PG}}^b,\, \bar{H})\bigr)}{\Tr\bigl(\Var(\bar{H})\bigr)},
    \label{eq:alpha1_star}
\end{equation}
and the resulting variance satisfies
\begin{equation}
    \Tr\bigl(\Var(g_{\mathrm{OHR}})\bigr) = \Tr\bigl(\Var(g_{\mathrm{PG}}^b)\bigr) - \frac{\bigl[\Tr\bigl(\Cov(g_{\mathrm{PG}}^b, \bar{H})\bigr)\bigr]^2}{\Tr\bigl(\Var(\bar{H})\bigr)} \;\leq\; \Tr\bigl(\Var(g_{\mathrm{PG}}^b)\bigr).
    \label{eq:ohr_variance_bound}
\end{equation}
\end{theorem}

\begin{proof}[Proof sketch]
The trace of variance is a convex quadratic in $\alpha_1$. Setting its derivative to zero yields \Cref{eq:alpha1_star}. The variance reduction follows from the non-negativity of the subtracted ratio. See \Cref{app:alpha1_derivation} for the full derivation.
\end{proof}

\paragraph{Practical estimation.}
During training, we estimate the moments in \Cref{eq:alpha1_star} from the current batch, but center $H$ using an exponential moving average (EMA) of batch means across optimizer steps. Writing $\hat{\mu}_{H,k}$ for this running estimate, the batch-averaged gradient estimate at optimizer step $k$ is
\begin{equation*}
    \hat{g}_{\mathrm{OHR},k} = \frac{1}{M}\sum_{m=1}^{M}\Bigl[g_{\mathrm{PG},k}^{b,(m)} + \hat{\alpha}_{1,k}\bigl(H_k^{(m)} - \hat{\mu}_{H,k}\bigr)\Bigr].
\end{equation*}
Centering instead with the current batch mean $\frac{1}{M}\sum_m H_k^{(m)}$ would make the batch-averaged correction identically zero, recovering the PG update. The EMA-based estimator approximates population OHR; estimating the mean and coefficient does not, in general, preserve exact unbiasedness or the population variance-reduction guarantee. Implementation details are provided in \Cref{app:experiment_details}.

Our frozen-policy diagnostics estimate per-sample variances from large batches. These measurements are distinct from the variance of the batch-averaged training update, for which same-batch centering would cancel the correction.

\begin{remark}
The optimal coefficient $\alpha_1^*$ minimizes the \emph{trace} of the variance, a scalar summary. It does not minimize variance component-wise: in high-dimensional parameter spaces, a diagonal or matrix-valued coefficient could yield further gains at the cost of additional estimation overhead.
\end{remark}

\begin{remark}[MSE-minimized variant]
\label{rem:phr}
When a small bias is acceptable, one can optimize both a scaling coefficient on $g_{\mathrm{PG}}^b$ and a coefficient on $H$ jointly to minimize MSE rather than variance, yielding a two-coefficient estimator with closed-form solution via a $2\times 2$ linear system. We call this \emph{PHR} and defer its derivation to \Cref{app:phr_derivation}.
\end{remark}

\subsection{When Does History Reparameterization Help?}
\label{sec:when_hr_helps}

With exact centering and the population-optimal coefficient $\alpha_1^*$, OHR cannot increase trace-variance and strictly reduces it when the trace-covariance between $g_{\mathrm{PG}}^b$ and $\bar{H}$ is nonzero. We now characterize \emph{when} this correlation is large.

\begin{proposition}[Variance reduction characterization]
\label{prop:var_reduction}
For $\alpha_1 = \alpha_1^*$ and exact centering, the fractional trace-variance reduction of OHR over PG is
\begin{equation}
    \frac{\Tr(\Var(g_{\mathrm{PG}}^b)) - \Tr(\Var(g_{\mathrm{OHR}}))}{\Tr(\Var(g_{\mathrm{PG}}^b))} \;=\; \rho^2(g_{\mathrm{PG}}^b,\, \bar{H}),
    \label{eq:frac_var_red}
\end{equation}
where $\rho^2$ is the squared trace-correlation between $g_{\mathrm{PG}}^b$ and $\bar{H}$, defined as $\rho^2 = \Tr(\Cov(g_{\mathrm{PG}}^b, \bar{H}))^2 / [\Tr(\Var(g_{\mathrm{PG}}^b))\,\Tr(\Var(\bar{H}))]$.
This equals zero when the history-pathwise term is uncorrelated with the policy gradient, and approaches one when they are perfectly correlated.
\end{proposition}

\begin{proof}
Immediate from \Cref{eq:ohr_variance_bound} and the optimality of $\alpha_1^*$ in \Cref{eq:alpha1_star}.
\end{proof}

What governs $\rho^2$ is the temporal structure encoded in the policy's transition dynamics. The history-pathwise correction $C$ propagates credit through the Jacobian $\partial \log p_\theta(x_t \mid x_{<t}) / \partial x_j$; when this Jacobian is structured (e.g., low-rank, as in a pretrained model with learned sparse dependencies), the correction is strongly correlated with $g_{\mathrm{PG}}^b$ and $\rho^2$ is large. When the Jacobian is unstructured (e.g., at random initialization), the correction is mostly noise and $\rho^2$ is small. This predicts that HR is most beneficial in the \emph{fine-tuning regime}, where a pretrained model has already learned sequential dependencies---a prediction we confirm experimentally (\Cref{sec:exp_tmaze}).

\begin{remark}[HR as self-model-based credit assignment]
\label{rem:self_model}
The history-pathwise term $C$ queries the policy's own Jacobian to propagate credit backward through the trajectory. This is structurally similar to model-based RL in which the policy uses its own dynamics model with no auxiliary model required.
\end{remark}

\paragraph{Computational cost.}
The history-pathwise term $C$ requires one backward pass through the reparameterized trajectory, adding $O(N)$ compute relative to REINFORCE. For OHR, the batch statistics for $\alpha_1^*$ add negligible overhead. The BPTT through history can be truncated to $K < N$ steps~\citep{williams1990truncated}, trading bias for reduced memory.

\section{Analytical Validation: Gaussian Linear Markov Model}
\label{sec:theory}

To build intuition and verify the theoretical properties of the HR estimator family, we study a setting where all quantities---gradient, bias, variance, and optimal coefficients---admit closed-form expressions.

\subsection{Setup}
\label{sec:theory_setup}

Consider a Gaussian linear Markov policy parameterized by $W \in \R^{d \times d}$:
\begin{equation}
    p_W(x_t \mid x_{t-1}) = \mathcal{N}(W x_{t-1},\, I),
    \label{eq:gaussian_policy}
\end{equation}
with a fixed initial state $x_{-1} \in \R^d$. The reparameterization is exact: $\chi_t = W \chi_{t-1} + \varepsilon_t$ with $\varepsilon_t \overset{\mathrm{iid}}{\sim} \mathcal{N}(0, I)$ and $\chi_t = x_t$ in value. The log-density (up to constants) is $\log p_W(x_t \mid x_{t-1}) = -\frac{1}{2}\norm{x_t - W x_{t-1}}^2$, so the frozen-history score at step $t$ is $\partial_W \log p_W(x_t \mid x_{t-1}) = \varepsilon_t x_{t-1}^\top$.

Unrolling the chain gives $x_t = W^{t+1}x_{-1} + \sum_{k=0}^{t} W^{t-k}\varepsilon_k$, and the estimators take the form:
\begin{align}
    g_{\mathrm{PG}} &= r(x_N) \sum_{t=0}^{N} \varepsilon_t\, x_{t-1}^\top, \label{eq:gauss_gpg} \\
    C &= \sum_{t=1}^{N}\sum_{k=0}^{t-1} (W^\top)^{t-k} \varepsilon_t\, x_{k-1}^\top, \label{eq:gauss_C} \\
    g_{\mathrm{HR}} &= g_{\mathrm{PG}} + r(x_N)\,C. \label{eq:gauss_ghr}
\end{align}

\subsection{Closed-Form Results}
\label{sec:theory_results}

We derive exact expressions for $J$, $\nabla_W J$, and $\mathrm{Bias}_{\mathrm{HR}}$ under two reward functions (linear $r = a^\top x_N$ and $L^2$ reward $r = -\frac{1}{2}\|x_N - y\|^2$; full expressions in \Cref{app:gaussian_formulas}). For $N=1$ with the linear reward, the bias takes a revealing form: $\mathrm{Bias}_{\mathrm{HR}} = W^\top a\,x_{-1}^\top$, which is precisely the history-pathwise component of $\nabla_W J = a(Wx_{-1})^\top + W^\top a\,x_{-1}^\top$. The base HR estimator captures this credit-assignment signal but double-counts it; the OHR correction resolves this by optimally weighting the history-pathwise component.

\paragraph{General $N$.}
For the linear reward $r = a^\top x_N$ and arbitrary horizon $N$, the exact gradient and HR estimator terms admit the unrolled expressions:
\begin{align}
    \nabla_W J &= \sum_{k=0}^{N} (W^\top)^k a\; (W^{N-k}x_{-1})^\top, \label{eq:gauss_grad_general} \\
    g_{\mathrm{PG}} &= (a^\top x_N) \sum_{t=0}^{N} \varepsilon_t\, x_{t-1}^\top, \label{eq:gauss_gpg_general} \\
    g_{\mathrm{HR}} &= g_{\mathrm{PG}} + (a^\top x_N)\sum_{t=1}^{N}\sum_{k=0}^{t-1} (W^\top)^{t-k}\varepsilon_t\, x_{k-1}^\top. \label{eq:gauss_ghr_general}
\end{align}
The bias grows with $N$ when $\rho(W) \geq 1$, due to the accumulation of $(W^\top)^{t-k}$ terms; when $\rho(W) < 1$ the terms decay geometrically and the bias norm converges. In either regime, the OHR correction is needed to control the bias--variance tradeoff.

\paragraph{Experimental Validation.}
Monte Carlo experiments ($M = 50{,}000$) confirm that the OHR variance reduction ratio grows monotonically with horizon, reaching $1.18\times$ at $N{=}20$ for the $L^2$ reward (\Cref{app:gaussian_formulas}). In the scalar case ($d=1$), the limiting variance reduction admits the closed form $\rho^2 \to w^2/(1+w^2)$ as $N \to \infty$ for $|w| < 1$.

\section{Experiments}
\label{sec:experiments}

We evaluate OHR in three settings of increasing complexity: a minimal discrete autoregressive model that establishes embedding-level reparameterization (\Cref{eq:bg_embedding_reparam}) as the effective relaxation strategy, a sparse-reward T-maze with a recurrent policy, and an autoregressive Transformer-style model for TSP. Full hyperparameter details and supplementary figures are in \Cref{app:experiment_details,app:supp_experiments}.

\subsection{Embedding-Based Reparameterization for Discrete Tokens}
\label{sec:exp_embedding}

\paragraph{Setup.}
We use a minimal autoregressive model ($V = 8$, $d_e = 16$, embedding layer + linear head) with a sparse terminal reward for matching a fixed target sequence. We compare ST-Gumbel (straight-through hard tokens) and continuous embedding reparameterization (\Cref{eq:bg_embedding_reparam}), sweeping $\tau \in \{1.0, 2.0, 3.0\}$ and horizons ranging from $3$ to $20$, measuring variance reduction with $M = 2{,}000$ per-sample gradients over 5 seeds.

\paragraph{Results.}
ST-Gumbel provides \emph{zero} variance reduction at all temperatures and horizons (ratio $\leq 1.01$): because hard one-hot tokens are used in the forward pass, the relaxed and true histories coincide and the correction vanishes (\Cref{sec:bg_discrete}). In contrast, continuous embedding reparameterization achieves substantial and growing variance reduction: from $1.5\times$ at $N{=}3$, $\tau{=}1$ to $\mathbf{4.9\times}$ at $N{=}15$, $\tau{=}3$ (\Cref{fig:embedding_hr}, left). \textbf{PG variance grows with $N$ while OHR variance remains approximately constant}, with higher temperature increasing the reduction (\Cref{fig:embedding_hr}, center). The optimal coefficient $\alpha_1^*$ is stable at ${\approx}1.1$--$1.3$ across all horizons, and the correction is anti-correlated with the policy gradient ($\cos \approx -0.6$), confirming that it acts as an effective control variate (\Cref{fig:embedding_hr}, right). Training dynamics confirming that this oracle VR translates to faster convergence are in \Cref{app:supp_embedding}.

\begin{figure}[t]
    \centering
    \includegraphics[width=\textwidth]{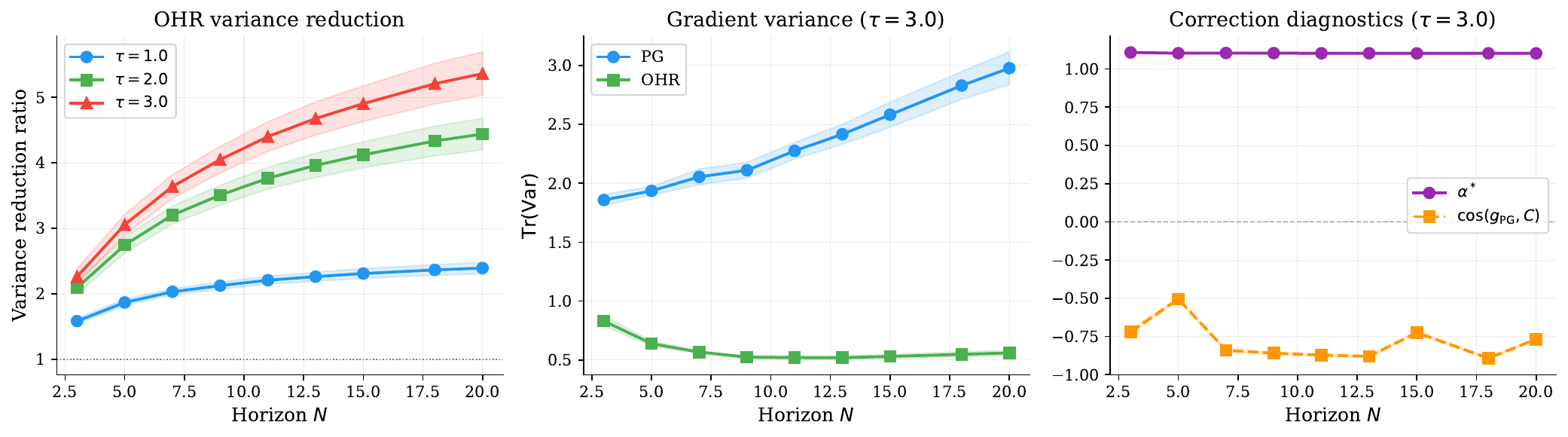}
    \caption{Embedding-based HR ($V=8$, $d_e=16$). Left: variance reduction ratio grows with horizon $N$ and temperature $\tau$. Center: PG variance (blue) grows with $N$ while OHR variance (green) remains flat. Right: $\alpha_1^* \approx 1.1$ is stable across horizons, and the correction is anti-correlated with PG ($\cos \approx -0.6$).}
    \label{fig:embedding_hr}
\end{figure}

\subsection{T-Maze: Credit Assignment with a Recurrent Policy}
\label{sec:exp_tmaze}

\paragraph{Setup.}
The T-maze~\citep{bakker2001tmaze} is a standard credit-assignment benchmark: a binary cue at $t{=}0$, $N{-}1$ uninformative observations, and a junction at $t{=}N$ where the agent must reproduce the cue (reward $+1$ correct, $0$ otherwise). The policy is a GRU~\citep{cho2014gru} ($h{=}16$) with 2-way categorical output. Actions feed back into the next GRU step via embedding reparameterization (\Cref{eq:bg_embedding_reparam}) with a learned action embedding $E_a \in \R^{2 \times h}$; in PG mode the feedback is detached. Following the prediction of \Cref{sec:when_hr_helps}, we supervised-pretrain each GRU to intermediate accuracy (${\sim}0.78$ reward) so the policy has learned temporal structure.

\paragraph{Oracle variance reduction.}
On frozen pretrained policies ($M{=}1{,}024$, 10 seeds), PG variance grows steeply with corridor length while OHR variance grows much slower, yielding a ratio that increases monotonically with~$N$: at $N{=}30$, OHR achieves $3.4\times$ variance reduction ($\tau{=}2.0$) and $3.1\times$ ($\tau{=}1.0$). Full oracle diagnostics are in \Cref{app:supp_tmaze}.

\paragraph{Training dynamics.}
\Cref{fig:tmaze_learning} compares PG and OHR training from identical initializations (10 seeds, $\tau{=}1.0$, $M{=}128$). At $N{=}10$, both methods solve the task reliably. At $N{=}15$, PG fails on some seeds while OHR succeeds on all. At $N{=}20$, PG succeeds on half the seeds while OHR solves most seeds (mean final reward $0.75$ vs.\ $0.90$). OHR maintains consistently lower estimator variance throughout training, with a $3.1\times$ ratio during active learning at $N{=}15$. The individual seed traces make the failure mode visible: PG seeds stuck at ${\sim}0.5$ correspond to cases where the high-variance estimator cannot propagate the sparse terminal signal back to the initial cue.

\begin{figure}[t]
    \centering
    \includegraphics[width=\textwidth]{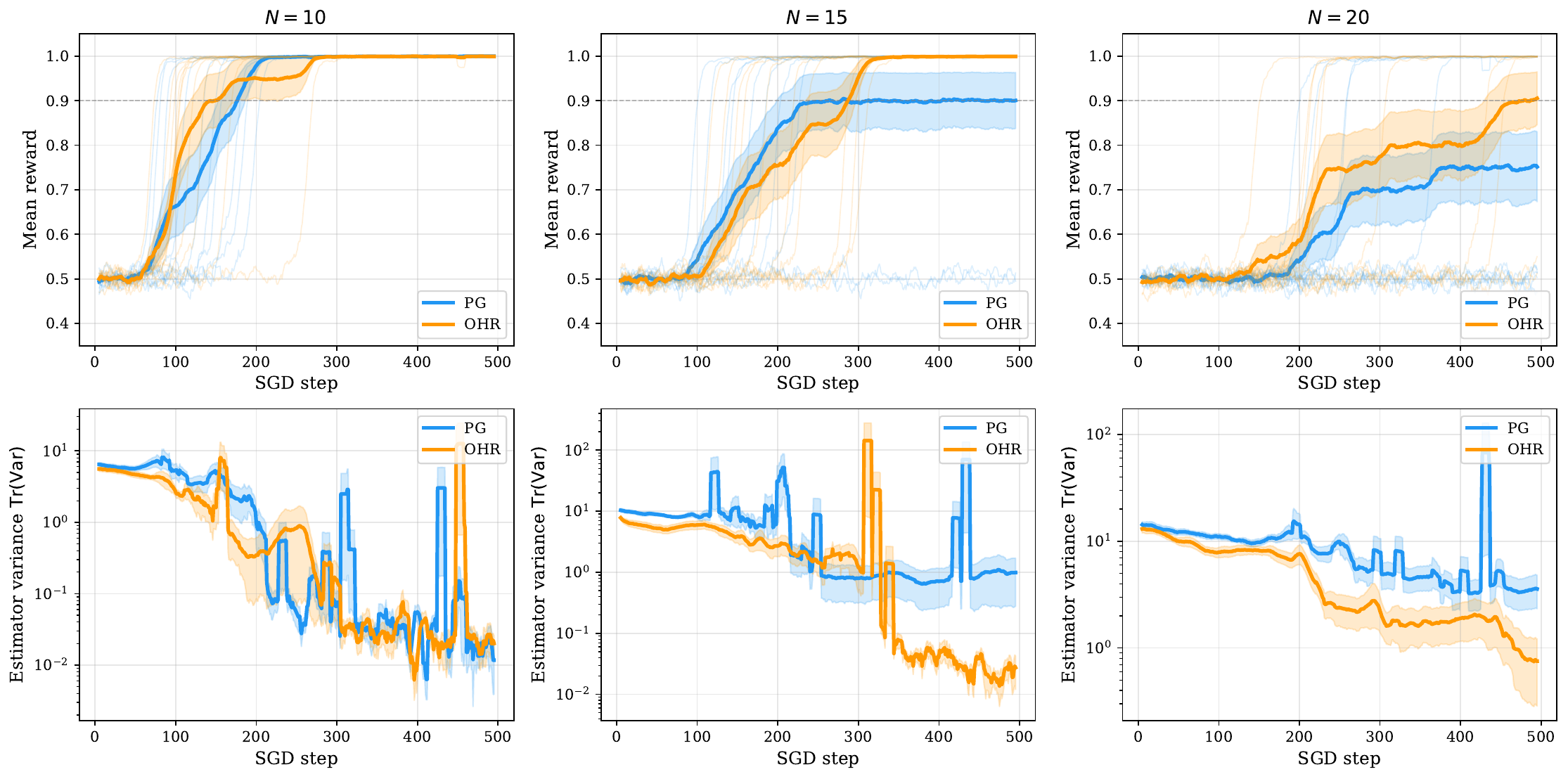}
    \caption{T-maze training: PG (blue) vs.\ OHR (orange), 10 seeds per $N$, $\tau{=}1.0$. Top: reward curves with individual seed traces (thin) and mean $\pm$ s.e.\ (thick). Bottom: estimator variance during training (log scale). As the corridor lengthens, PG increasingly fails to converge while OHR remains reliable.}
    \label{fig:tmaze_learning}
\end{figure}

\subsection{TSP with Autoregressive Construction}
\label{sec:exp_tsp}

We test HR on TSP, the canonical benchmark for autoregressive RL in combinatorial optimization~\citep{bello2017neural,kool2019attention}, where tour construction is inherently sequential and the reward is terminal. We use a simplified Attention Model~\citep{kool2019attention} (${\sim}120$K parameters; architecture and hyperparameters in \Cref{app:experiment_details}), replacing the hard embedding feedback with a soft mixture via \Cref{eq:bg_embedding_reparam}.

Both methods converge to similar final tour lengths, but OHR achieves $1.15$--$1.25\times$ lower gradient variance throughout training (see \Cref{fig:tsp_learning} in \Cref{app:supp_tsp}). The effect is smaller than for the embedding model ($4.9\times$) or T-maze ($3.4\times$), reflecting the attention decoder's sensitivity to context perturbations. The small batch size ($M{=}16$) is critical: at $M{=}128$ the per-sample VR has negligible impact, but at $M{=}16$ it translates to visibly smoother optimization.

\subsection{Composability with Reward-Side Methods}
\label{sec:exp_composability}

Since HR operates on the score function while reward-side methods (baselines, critics) operate on the advantage $(r-b)$, the two should compose. We verify this on the T-maze with a learned critic and on TSP with the RLOO baseline~\citep{kool2019buy}. On TSP, the reductions compose near-multiplicatively: at $n{=}10$, RLOO alone gives $2.45\times$, OHR alone gives $1.27\times$, and OHR + RLOO gives $\mathbf{3.55\times}$ (vs.\ a product prediction of $3.11\times$). On the T-maze, OHR reduces variance with a learned critic at all corridor lengths tested (up to $1.70\times$ at $N{=}15$). However, OHR does not reduce variance in every reported configuration: with a batch-mean baseline, the PG/OHR variance ratio is $0.86\times$ for the T-maze at $N{=}10$ and $0.94\times$ for TSP at $n{=}8$ (\Cref{tab:composability}). Full results and learning curves are in \Cref{app:supp_composability}.

\section{Related Work}
\label{sec:related_work}

HR is closest to control-variate methods for policy-gradient and discrete-gradient estimation. REBAR~\citep{tucker2017} and RELAX~\citep{grathwohl2018} construct control variates by reparameterizing the reward or a learned surrogate; HR instead reparameterizes the conditioning history in the score function, exploiting autoregressive structure without requiring a differentiable reward. Q-Prop~\citep{gu2017qprop} uses reparameterized critic gradients as control variates in continuous-action RL, but relies on an off-policy critic and environment-provided states, whereas HR uses the policy's own autoregressive dynamics.

\vspace{-0.5em}

\paragraph{Stochastic computation graphs.}
\citet{schulman2015scg} provide the general framework for combining score-function and pathwise estimators in computation graphs with stochastic and deterministic nodes; DiCE~\citep{foerster2018dice} extends this to higher-order estimators. HR can be viewed as upgrading autoregressive history edges from detached stochastic nodes to reparameterized deterministic nodes, exposing additional pathwise gradient paths.

\vspace{-0.5em}

\paragraph{Credit assignment in sequences.}
RL methods for language models, including PPO~\citep{schulman2017ppo} and GRPO~\citep{shao2024grpo}, typically rely on sequence-level score-function gradients with learned or group-relative baselines. Process reward models~\citep{lightman2024prm} address credit assignment by adding per-step supervision. HR is complementary: process rewards improve the advantage signal $(r-b)$, while HR reduces variance on the score-function side $\nabla_\theta \log p_\theta$.

We discuss reparameterization gradients, discrete gradient estimators, world models~\citep{hafner2020dreamer,hafner2023dreamerv3}, and additional positioning in \Cref{app:extended_related_work}.

\section{Limitations}
\label{sec:limitations}

The evaluation spans low-dimensional settings: Gaussian models ($d \leq 10$), a minimal autoregressive model ($V = 8$), a T-maze with a GRU ($N \leq 30$), and a small attention model for TSP (${\sim}120$K parameters). We do not evaluate on natural-language tasks, so the present results do not establish effectiveness for language-model reinforcement learning. Whether HR remains beneficial in high-dimensional models with large parameter counts remains to be tested. We have not tested HR with learned per-token baselines or PPO-style value functions in large autoregressive models. The temporal structure that HR exploits is characterized informally; a formal measure analogous to the scalar $w^2/(1+w^2)$ result for general architectures is an open question.
\section{Conclusion}
\label{sec:conclusion}

We introduced History-Reparameterized policy gradients, which inject pathwise credit-assignment information into the score function by reparameterizing the conditioning history rather than the reward. With exact centering and the population-optimal coefficient, OHR is unbiased and cannot increase the trace of the gradient variance. Across all settings tested, embedding-based relaxation yields variance reduction that grows monotonically with horizon. The main open question is whether these gains persist in larger models with high-dimensional parameter spaces and learned baselines.

\begin{ack}
Mohammad Hossein thanks Simone Bombari, Aryo Lotfi, and Samy Bengio for fruitful discussions.
Robert West's lab is partly supported by a Swiss National Science Foundation Starting Grant (TMSGI2\_211379).
The authors declare no competing interests.
\end{ack}

\bibliographystyle{plainnat}
\bibliography{references}

\begin{thebibliography}{41}
\providecommand{\natexlab}[1]{#1}
\providecommand{\url}[1]{\texttt{#1}}
\expandafter\ifx\csname urlstyle\endcsname\relax
  \providecommand{\doi}[1]{doi: #1}\else
  \providecommand{\doi}{doi: \begingroup \urlstyle{rm}\Url}\fi

\bibitem[Amani et~al.(2024)Amani, Baldwin, Mansouri, Josifoski, Peyrard, and
  West]{amani2024sigmaae}
Mohammad~Hossein Amani, Nicolas~Mario Baldwin, Amin Mansouri, Martin Josifoski,
  Maxime Peyrard, and Robert West.
\newblock Symbolic autoencoding for self-supervised sequence learning.
\newblock \emph{arXiv preprint arXiv:2402.10575}, 2024.
\newblock URL \url{https://arxiv.org/abs/2402.10575}.

\bibitem[Amani et~al.(2026)Amani, Lotfi, Baldwin, Bengio, Farajtabar, Abbe, and
  West]{amani2026}
Mohammad~Hossein Amani, Aryo Lotfi, Nicolas~Mario Baldwin, Samy Bengio, Mehrdad
  Farajtabar, Emmanuel Abbe, and Robert West.
\newblock {RL} for reasoning by adaptively revealing rationales.
\newblock In \emph{International Conference on Learning Representations
  (ICLR)}, 2026.
\newblock URL \url{https://arxiv.org/abs/2506.18110}.

\bibitem[Bakker(2002)]{bakker2001tmaze}
Bram Bakker.
\newblock Reinforcement learning with long short-term memory.
\newblock \emph{Advances in Neural Information Processing Systems}, 14, 2002.

\bibitem[Bello et~al.(2017)Bello, Pham, Le, Norouzi, and
  Bengio]{bello2017neural}
Irwan Bello, Hieu Pham, Quoc~V. Le, Mohammad Norouzi, and Samy Bengio.
\newblock Neural combinatorial optimization with reinforcement learning.
\newblock In \emph{International Conference on Learning Representations (ICLR),
  Workshop Track}, 2017.
\newblock URL \url{https://arxiv.org/abs/1611.09940}.

\bibitem[Bengio et~al.(2013)Bengio, L{\'e}onard, and Courville]{bengio2013}
Yoshua Bengio, Nicholas L{\'e}onard, and Aaron Courville.
\newblock Estimating or propagating gradients through stochastic neurons for
  conditional computation.
\newblock \emph{arXiv preprint arXiv:1308.3432}, 2013.
\newblock URL \url{https://arxiv.org/abs/1308.3432}.

\bibitem[Chernozhukov et~al.(2018)Chernozhukov, Chetverikov, Demirer, Duflo,
  Hansen, Newey, and Robins]{chernozhukov2018dml}
Victor Chernozhukov, Denis Chetverikov, Mert Demirer, Esther Duflo, Christian
  Hansen, Whitney Newey, and James Robins.
\newblock Double/debiased machine learning for treatment and structural
  parameters.
\newblock \emph{The Econometrics Journal}, 21\penalty0 (1):\penalty0 C1--C68,
  2018.

\bibitem[Cho et~al.(2014)Cho, van Merri{\"e}nboer, Gulcehre, Bahdanau,
  Bougares, Schwenk, and Bengio]{cho2014gru}
Kyunghyun Cho, Bart van Merri{\"e}nboer, Caglar Gulcehre, Dzmitry Bahdanau,
  Fethi Bougares, Holger Schwenk, and Yoshua Bengio.
\newblock Learning phrase representations using {RNN} encoder--decoder for
  statistical machine translation.
\newblock In \emph{Proceedings of the 2014 Conference on Empirical Methods in
  Natural Language Processing (EMNLP)}, pages 1724--1734, 2014.

\bibitem[Dong et~al.(2020)Dong, Mnih, and Tucker]{dong2020disarm}
Zhe Dong, Andriy Mnih, and George Tucker.
\newblock {DisARM}: An antithetic gradient estimator for binary latent
  variables.
\newblock In \emph{Advances in Neural Information Processing Systems}, 2020.
\newblock URL \url{https://arxiv.org/abs/2006.10680}.

\bibitem[Fan et~al.(2012)Fan, Guo, and Yu]{fan2012refitted}
Jianqing Fan, Shaojun Guo, and Ke~Yu.
\newblock Variance estimation using refitted cross-validation in ultrahigh
  dimensional regression.
\newblock \emph{Journal of the Royal Statistical Society, Series B},
  74:\penalty0 37--65, 2012.

\bibitem[Figurnov et~al.(2018)Figurnov, Mohamed, and Mnih]{figurnov2018}
Mikhail Figurnov, Shakir Mohamed, and Andriy Mnih.
\newblock Implicit reparameterization gradients.
\newblock In \emph{Advances in Neural Information Processing Systems}, 2018.
\newblock URL \url{https://arxiv.org/abs/1805.08498}.

\bibitem[Florensa et~al.(2017)Florensa, Held, Wulfmeier, Zhang, and
  Abbeel]{florensa2017reverse}
Carlos Florensa, David Held, Markus Wulfmeier, Michael Zhang, and Pieter
  Abbeel.
\newblock Reverse curriculum generation for reinforcement learning.
\newblock In \emph{Conference on Robot Learning (CoRL)}, pages 482--495, 2017.
\newblock URL \url{https://proceedings.mlr.press/v78/florensa17a.html}.

\bibitem[Foerster et~al.(2018)Foerster, Farquhar, Al-Shedivat, Rockt{\"a}schel,
  Xing, and Whiteson]{foerster2018dice}
Jakob Foerster, Gregory Farquhar, Maruan Al-Shedivat, Tim Rockt{\"a}schel,
  Eric~P. Xing, and Shimon Whiteson.
\newblock {DiCE}: The infinitely differentiable {Monte Carlo} estimator.
\newblock In \emph{International Conference on Machine Learning (ICML)}, 2018.
\newblock URL \url{https://arxiv.org/abs/1802.05098}.

\bibitem[Grathwohl et~al.(2018)Grathwohl, Choi, Wu, Roeder, and
  Duvenaud]{grathwohl2018}
Will Grathwohl, Dami Choi, Yuhuai Wu, Geoffrey Roeder, and David Duvenaud.
\newblock Backpropagation through the void: Optimizing control variates for
  black-box gradient estimation.
\newblock In \emph{International Conference on Learning Representations
  (ICLR)}, 2018.
\newblock URL \url{https://arxiv.org/abs/1711.00123}.

\bibitem[Gu et~al.(2017)Gu, Lillicrap, Ghahramani, Turner, and
  Levine]{gu2017qprop}
Shixiang Gu, Timothy Lillicrap, Zoubin Ghahramani, Richard~E. Turner, and
  Sergey Levine.
\newblock {Q-Prop}: Sample-efficient policy gradient with an off-policy critic.
\newblock In \emph{International Conference on Learning Representations
  (ICLR)}, 2017.

\bibitem[Hafner et~al.(2020)Hafner, Lillicrap, Ba, and
  Norouzi]{hafner2020dreamer}
Danijar Hafner, Timothy Lillicrap, Jimmy Ba, and Mohammad Norouzi.
\newblock Dream to control: Learning behaviors by latent imagination.
\newblock In \emph{International Conference on Learning Representations
  (ICLR)}, 2020.
\newblock URL \url{https://arxiv.org/abs/1912.01603}.

\bibitem[Hafner et~al.(2023)Hafner, Pasukonis, Ba, and
  Lillicrap]{hafner2023dreamerv3}
Danijar Hafner, Jurgis Pasukonis, Jimmy Ba, and Timothy Lillicrap.
\newblock Mastering diverse domains through world models.
\newblock \emph{arXiv preprint arXiv:2301.04104}, 2023.
\newblock URL \url{https://arxiv.org/abs/2301.04104}.

\bibitem[Jang et~al.(2017)Jang, Gu, and Poole]{jang2017}
Eric Jang, Shixiang Gu, and Ben Poole.
\newblock Categorical reparameterization with {Gumbel-Softmax}.
\newblock In \emph{International Conference on Learning Representations
  (ICLR)}, 2017.
\newblock URL \url{https://arxiv.org/abs/1611.01144}.

\bibitem[Jankowiak and Obermeyer(2018)]{jankowiak2018}
Martin Jankowiak and Fritz Obermeyer.
\newblock Pathwise derivatives beyond the reparameterization trick.
\newblock In \emph{International Conference on Machine Learning (ICML)}, 2018.
\newblock URL \url{https://arxiv.org/abs/1806.01851}.

\bibitem[Kingma and Welling(2014)]{kingma2014}
Diederik~P. Kingma and Max Welling.
\newblock Auto-encoding variational {Bayes}.
\newblock In \emph{International Conference on Learning Representations
  (ICLR)}, 2014.
\newblock URL \url{https://arxiv.org/abs/1312.6114}.

\bibitem[Kool et~al.(2019{\natexlab{a}})Kool, van Hoof, and
  Welling]{kool2019attention}
Wouter Kool, Herke van Hoof, and Max Welling.
\newblock Attention, learn to solve routing problems!
\newblock In \emph{International Conference on Learning Representations
  (ICLR)}, 2019{\natexlab{a}}.
\newblock URL \url{https://arxiv.org/abs/1803.08475}.

\bibitem[Kool et~al.(2019{\natexlab{b}})Kool, van Hoof, and
  Welling]{kool2019buy}
Wouter Kool, Herke van Hoof, and Max Welling.
\newblock Buy 4 {REINFORCE} samples, get a baseline for free!
\newblock In \emph{Deep Reinforcement Learning Meets Structured Prediction
  Workshop at ICLR}, 2019{\natexlab{b}}.

\bibitem[Kwon et~al.(2020)Kwon, Choo, Kim, Yoon, Gwon, and Min]{kwon2020pomo}
Yeong-Dae Kwon, Jinho Choo, Byoungjip Kim, Iljoo Yoon, Youngjune Gwon, and
  Seungjai Min.
\newblock {POMO}: Policy optimization with multiple optima for reinforcement
  learning.
\newblock In \emph{Advances in Neural Information Processing Systems}, 2020.
\newblock URL \url{https://arxiv.org/abs/2010.16011}.

\bibitem[Lightman et~al.(2024)Lightman, Kosaraju, Burda, Edwards, Baker, Lee,
  Leike, Schulman, Sutskever, and Cobbe]{lightman2024prm}
Hunter Lightman, Vineet Kosaraju, Yuri Burda, Harri Edwards, Bowen Baker, Teddy
  Lee, Jan Leike, John Schulman, Ilya Sutskever, and Karl Cobbe.
\newblock Let's verify step by step.
\newblock In \emph{International Conference on Learning Representations
  (ICLR)}, 2024.
\newblock URL \url{https://arxiv.org/abs/2305.20050}.

\bibitem[Maddison et~al.(2017)Maddison, Mnih, and Teh]{maddison2017}
Chris~J. Maddison, Andriy Mnih, and Yee~Whye Teh.
\newblock The {Concrete} distribution: A continuous relaxation of discrete
  random variables.
\newblock In \emph{International Conference on Learning Representations
  (ICLR)}, 2017.
\newblock URL \url{https://arxiv.org/abs/1611.00712}.

\bibitem[Mnih and Gregor(2014)]{mnih2014nvil}
Andriy Mnih and Karol Gregor.
\newblock Neural variational inference and learning in belief networks.
\newblock In \emph{International Conference on Machine Learning (ICML)}, 2014.

\bibitem[Naesseth et~al.(2017)Naesseth, Ruiz, Linderman, and
  Blei]{naesseth2016}
Christian~A. Naesseth, Francisco J.~R. Ruiz, Scott~W. Linderman, and David~M.
  Blei.
\newblock Reparameterization gradients through acceptance-rejection sampling
  algorithms.
\newblock In \emph{International Conference on Artificial Intelligence and
  Statistics (AISTATS)}, 2017.
\newblock URL \url{https://arxiv.org/abs/1610.05683}.

\bibitem[Paulus et~al.(2021)Paulus, Maddison, and Krause]{paulus2021rao}
Max~B. Paulus, Chris~J. Maddison, and Andreas Krause.
\newblock Rao-{Blackwellizing} the straight-through {Gumbel-Softmax} gradient
  estimator.
\newblock In \emph{International Conference on Learning Representations
  (ICLR)}, 2021.
\newblock URL \url{https://arxiv.org/abs/2010.04838}.

\bibitem[Ranganath et~al.(2014)Ranganath, Gerrish, and Blei]{ranganath2014bbvi}
Rajesh Ranganath, Sean Gerrish, and David~M. Blei.
\newblock Black box variational inference.
\newblock In \emph{International Conference on Artificial Intelligence and
  Statistics (AISTATS)}, 2014.

\bibitem[Rezende et~al.(2014)Rezende, Mohamed, and
  Wierstra]{rezende2014stochastic}
Danilo~Jimenez Rezende, Shakir Mohamed, and Daan Wierstra.
\newblock Stochastic backpropagation and approximate inference in deep
  generative models.
\newblock In \emph{International Conference on Machine Learning (ICML)}, 2014.

\bibitem[Ruiz et~al.(2016)Ruiz, Titsias, and Blei]{ruiz2016}
Francisco J.~R. Ruiz, Michalis~K. Titsias, and David~M. Blei.
\newblock The generalized reparameterization gradient.
\newblock In \emph{Advances in Neural Information Processing Systems}, 2016.
\newblock URL \url{https://arxiv.org/abs/1610.02287}.

\bibitem[Salimans and Chen(2018)]{salimans2018montezuma}
Tim Salimans and Richard Chen.
\newblock Learning {Montezuma's Revenge} from a single demonstration.
\newblock \emph{arXiv preprint arXiv:1812.03381}, 2018.
\newblock URL \url{https://arxiv.org/abs/1812.03381}.

\bibitem[Schulman et~al.(2015)Schulman, Heess, Weber, and
  Abbeel]{schulman2015scg}
John Schulman, Nicolas Heess, Theophane Weber, and Pieter Abbeel.
\newblock Gradient estimation using stochastic computation graphs.
\newblock In \emph{Advances in Neural Information Processing Systems}, 2015.
\newblock URL \url{https://arxiv.org/abs/1506.05254}.

\bibitem[Schulman et~al.(2016)Schulman, Moritz, Levine, Jordan, and
  Abbeel]{schulman2015}
John Schulman, Philipp Moritz, Sergey Levine, Michael~I. Jordan, and Pieter
  Abbeel.
\newblock High-dimensional continuous control using generalized advantage
  estimation.
\newblock In \emph{International Conference on Learning Representations
  (ICLR)}, 2016.
\newblock URL \url{https://arxiv.org/abs/1506.02438}.

\bibitem[Schulman et~al.(2017)Schulman, Wolski, Dhariwal, Radford, and
  Klimov]{schulman2017ppo}
John Schulman, Filip Wolski, Prafulla Dhariwal, Alec Radford, and Oleg Klimov.
\newblock Proximal policy optimization algorithms.
\newblock \emph{arXiv preprint arXiv:1707.06347}, 2017.
\newblock URL \url{https://arxiv.org/abs/1707.06347}.

\bibitem[Shao et~al.(2024)Shao, Wang, Zhu, et~al.]{shao2024grpo}
Zhihong Shao, Peiyi Wang, Qihao Zhu, et~al.
\newblock {DeepSeekMath}: Pushing the limits of mathematical reasoning in open
  language models.
\newblock \emph{arXiv preprint arXiv:2402.03300}, 2024.
\newblock URL \url{https://arxiv.org/abs/2402.03300}.

\bibitem[Tucker et~al.(2017)Tucker, Mnih, Maddison, Lawson, and
  Sohl-Dickstein]{tucker2017}
George Tucker, Andriy Mnih, Chris~J. Maddison, Dieterich Lawson, and Jascha
  Sohl-Dickstein.
\newblock {REBAR}: Low-variance, unbiased gradient estimates for discrete
  latent variable models.
\newblock In \emph{Advances in Neural Information Processing Systems}, 2017.
\newblock URL \url{https://arxiv.org/abs/1703.07370}.

\bibitem[Williams(1992)]{williams1992}
Ronald~J. Williams.
\newblock Simple statistical gradient-following algorithms for connectionist
  reinforcement learning.
\newblock \emph{Machine Learning}, 8:\penalty0 229--256, 1992.

\bibitem[Williams and Peng(1990)]{williams1990truncated}
Ronald~J. Williams and Jing Peng.
\newblock An efficient gradient-based algorithm for on-line training of
  recurrent network trajectories.
\newblock \emph{Neural Computation}, 2\penalty0 (4):\penalty0 490--501, 1990.

\bibitem[Xi et~al.(2024)Xi, Chen, Hong, Jin, Zheng, He, Ding, Liu, Guo, Wang,
  Guo, Shen, Fan, Zhou, Dou, Wang, Zhang, Sun, Gui, Zhang, and Huang]{xi2024r3}
Zhiheng Xi, Wenxiang Chen, Boyang Hong, Senjie Jin, Rui Zheng, Wei He, Yiwen
  Ding, Shichun Liu, Xin Guo, Junzhe Wang, Honglin Guo, Wei Shen, Xiaoran Fan,
  Yuhao Zhou, Shihan Dou, Xiao Wang, Xinbo Zhang, Peng Sun, Tao Gui, Qi~Zhang,
  and Xuanjing Huang.
\newblock Training large language models for reasoning through reverse
  curriculum reinforcement learning.
\newblock In \emph{International Conference on Machine Learning (ICML)}, pages
  54030--54048, 2024.
\newblock URL \url{https://proceedings.mlr.press/v235/xi24a.html}.

\bibitem[Yin and Zhou(2019)]{yin2019arm}
Mingzhang Yin and Mingyuan Zhou.
\newblock {ARM}: Augment-{REINFORCE}-merge gradient for stochastic binary
  networks.
\newblock In \emph{International Conference on Learning Representations
  (ICLR)}, 2019.
\newblock URL \url{https://arxiv.org/abs/1807.11143}.

\bibitem[Yin et~al.(2019)Yin, Yue, and Zhou]{yin2019arsm}
Mingzhang Yin, Yuguang Yue, and Mingyuan Zhou.
\newblock {ARSM}: Augment-{REINFORCE}-swap-merge estimator for gradient
  backpropagation through categorical variables.
\newblock In \emph{International Conference on Machine Learning (ICML)}, 2019.
\newblock URL \url{https://proceedings.mlr.press/v97/yin19c.html}.

\end{thebibliography}

\newpage
\appendix

\section{Proof of \texorpdfstring{\Cref{prop:bias}}{Proposition 1}: Bias Identity}
\label{app:bias_proof}

We provide the full proof of the bias characterization for the HR estimator.

\begin{proof}
Recall that the HR estimator is defined as
\begin{equation*}
    g_{\mathrm{HR}} = (r - b)\sum_{t=0}^{N}\nabla_\theta \log p_\theta(x_t \mid \chi_{<t}(\theta)).
\end{equation*}
Since $g_{\mathrm{PG}}^b = (r-b)\sum_t \nabla_\theta \log p_\theta(x_t \mid x_{<t})$ is unbiased with $\E[g_{\mathrm{PG}}^b] = \nabla_\theta J$, we write:
\begin{align*}
    \E[g_{\mathrm{HR}}] &= \E\Bigl[(r-b)\sum_{t=0}^{N}\nabla_\theta \log p_\theta(x_t \mid \chi_{<t}(\theta))\Bigr] \\
    &= \E\Bigl[(r-b)\sum_{t=0}^{N}\nabla_\theta \log p_\theta(x_t \mid x_{<t})\Bigr] \\
    &\quad + \E\Bigl[(r-b)\sum_{t=0}^{N}\bigl(\nabla_\theta \log p_\theta(x_t \mid \chi_{<t}(\theta)) - \nabla_\theta \log p_\theta(x_t \mid x_{<t})\bigr)\Bigr] \\
    &= \nabla_\theta J + \sum_{t=0}^{N}\E\Bigl[(r-b)\bigl(\nabla_\theta \log p_\theta(x_t \mid \chi_{<t}(\theta)) - \nabla_\theta \log p_\theta(x_t \mid x_{<t})\bigr)\Bigr].
\end{align*}

For exact history reparameterizations (e.g., continuous actions), the reparameterized and sampled histories coincide in value: $\chi_{<t}(\theta) = x_{<t}$. The two log-probabilities $\log p_\theta(x_t \mid \chi_{<t}(\theta))$ and $\log p_\theta(x_t \mid x_{<t})$ are then equal as scalar quantities. Their \emph{gradients} with respect to $\theta$ differ, however, because $\nabla_\theta \log p_\theta(x_t \mid \chi_{<t}(\theta))$ includes the additional chain-rule terms $\sum_{j<t}(\nabla_\theta \chi_j)^\top \frac{\partial \log p_\theta}{\partial \chi_j}$ from differentiating through the reparameterized history. These additional terms constitute the history-pathwise correction $C_t$ from \Cref{eq:g_hr_decomp}. For discrete relaxations where $\chi_{<t} \neq x_{<t}$ in value, the same decomposition applies but the log-probabilities may also differ; see \Cref{app:discrete_relaxation_caveat}.

Identifying the bias term with $\E[H] = \E[(r-b)\,C]$ completes the proof.
\end{proof}

\section{Continuous vs.\ Discrete History Reparameterization}
\label{app:discrete_relaxation_caveat}

The theory in \Cref{sec:bias} assumes that the reparameterized history satisfies $\chi_{<t}(\theta) = x_{<t}$ in value, so that the log-probability $\log p_\theta(x_t \mid \chi_{<t})$ equals $\log p_\theta(x_t \mid x_{<t})$ pointwise and only their total derivatives with respect to $\theta$ differ. This holds exactly for continuous reparameterizations (e.g., the Gaussian location-scale family), where $\chi_t = S(\phi_\theta(\chi_{<t}), \varepsilon_t) = x_t$.

For discrete tokens with embedding-level relaxation (\Cref{eq:bg_embedding_reparam}), the situation is different. The relaxed history $\chi_t = \softmax(\phi_\theta(\chi_{<t})/\tau + g_t) \cdot E$ is a weighted average of embedding rows, which does not equal the hard embedding $\onehot(x_t) \cdot E$ unless the softmax output is degenerate. Consequently, $\log p_\theta(x_t \mid \chi_{<t})$ and $\log p_\theta(x_t \mid x_{<t})$ may differ not only in gradient but also in value.

The OHR estimator (\Cref{eq:g_ohr}) remains valid in this setting because it uses $\bar{H} = H - \E[H]$ as a control variate: at the population level, any zero-mean random variable produces an unbiased estimator when added to $g_{\mathrm{PG}}^b$, regardless of how it is constructed. The variance-reduction guarantee of \Cref{thm:ohr_variance} likewise holds. What changes is the interpretation: for continuous actions, the HR correction is an exact history-pathwise gradient; for discrete tokens with embedding relaxation, it is a relaxation-based control variate that queries the policy's response to soft perturbations of the conditioning history. The bias characterization of \Cref{prop:bias} applies to the raw $g_{\mathrm{HR}}$ estimator in both cases, but the ``bias'' term has a larger magnitude in the discrete case because it includes the discrepancy from evaluating the log-probability at a different input, not just from differentiating through the history.

\section{Derivation of Optimal \texorpdfstring{$\alpha_1^*$}{alpha\_1*} (OHR)}
\label{app:alpha1_derivation}

We minimize the trace of variance of the OHR estimator $g_{\mathrm{OHR}} = g_{\mathrm{PG}}^b + \alpha_1 \bar{H}$ over $\alpha_1 \in \R$, where $\bar{H} = H - \E[H]$ is the centered history-pathwise term with $\E[\bar{H}] = 0$.

\begin{proof}
Define $g_c = g_{\mathrm{PG}}^b - \nabla_\theta J$ (centered PG estimator). Since $\E[\bar{H}] = 0$, the OHR estimator has mean $\E[g_{\mathrm{OHR}}] = \nabla_\theta J$ for any $\alpha_1$, and:
\begin{align*}
    \Tr(\Var(g_{\mathrm{OHR}})) &= \E\bigl[\norm{g_c + \alpha_1 \bar{H}}_F^2\bigr] \\
    &= \E[\norm{g_c}_F^2] + \alpha_1^2\,\E[\norm{\bar{H}}_F^2] + 2\alpha_1\,\E[\Tr(g_c^\top \bar{H})] \\
    &= \Tr(\Var(g_{\mathrm{PG}}^b)) + \alpha_1^2\,\Tr(\Var(\bar{H})) + 2\alpha_1\,\Tr(\Cov(g_{\mathrm{PG}}^b, \bar{H})).
\end{align*}

This is a convex quadratic in $\alpha_1$. The coefficient of $\alpha_1^2$ is $\Tr(\Var(\bar{H})) > 0$ for non-degenerate $\bar{H}$ (if $\bar{H} = 0$ a.s., the history-pathwise correction is constant and there is nothing to optimize). Setting the derivative to zero:
\begin{equation*}
    \frac{\partial}{\partial \alpha_1}\Tr(\Var(g_{\mathrm{OHR}})) = 2\alpha_1\,\Tr(\Var(\bar{H})) + 2\,\Tr(\Cov(g_{\mathrm{PG}}^b, \bar{H})) = 0,
\end{equation*}
yielding
\begin{equation*}
    \alpha_1^* = -\frac{\Tr(\Cov(g_{\mathrm{PG}}^b, \bar{H}))}{\Tr(\Var(\bar{H}))}.
\end{equation*}

Substituting back and simplifying:
\begin{align*}
    \Tr(\Var(g_{\mathrm{OHR}})) &= \Tr(\Var(g_{\mathrm{PG}}^b)) + \frac{[\Tr(\Cov(g_{\mathrm{PG}}^b, \bar{H}))]^2}{\Tr(\Var(\bar{H}))} - \frac{2[\Tr(\Cov(g_{\mathrm{PG}}^b, \bar{H}))]^2}{\Tr(\Var(\bar{H}))} \\
    &= \Tr(\Var(g_{\mathrm{PG}}^b)) - \frac{[\Tr(\Cov(g_{\mathrm{PG}}^b, \bar{H}))]^2}{\Tr(\Var(\bar{H}))} \\
    &\leq \Tr(\Var(g_{\mathrm{PG}}^b)),
\end{align*}
where the inequality follows from the non-negativity of the subtracted term. Equality holds if and only if $\Tr(\Cov(g_{\mathrm{PG}}^b, \bar{H})) = 0$, i.e., when the history-pathwise correction is uncorrelated with the policy gradient in the trace inner product.
\end{proof}

\section{PHR Optimal Coefficients}
\label{app:phr_derivation}

We derive the MSE-optimal coefficients $(\alpha_0^*, \alpha_1^*)$ for the PHR estimator $g_{\mathrm{PHR}} = \alpha_0\,g_{\mathrm{PG}}^b + \alpha_1\,H$.

\begin{proof}
The MSE objective is:
\begin{align*}
    f(\alpha_0, \alpha_1) &= \E[\norm{\alpha_0\,g_{\mathrm{PG}}^b + \alpha_1\,H - \nabla_\theta J}^2] \\
    &= \alpha_0^2 A + \alpha_1^2 B + \norm{\nabla_\theta J}^2 + 2\alpha_0\alpha_1 D - 2\alpha_0 E - 2\alpha_1 F,
\end{align*}
where $A = \E[\norm{g_{\mathrm{PG}}^b}^2]$, $B = \E[\norm{H}^2]$, $D = \E[{g_{\mathrm{PG}}^b}^\top H]$, $E = \E[{g_{\mathrm{PG}}^b}^\top \nabla_\theta J]$, and $F = \E[H^\top \nabla_\theta J]$.

This is a convex quadratic in $(\alpha_0, \alpha_1)$. The Hessian is $2\begin{psmallmatrix} A & D \\ D & B \end{psmallmatrix}$, which is positive semidefinite since $A, B \geq 0$ and $AB \geq D^2$ by the Cauchy--Schwarz inequality applied to the inner product $\E[u^\top v]$ with $u = g_{\mathrm{PG}}^b$ and $v = H$.

Setting the gradient to zero:
\begin{align*}
    \frac{\partial f}{\partial \alpha_0} &= 2\alpha_0 A + 2\alpha_1 D - 2E = 0, \\
    \frac{\partial f}{\partial \alpha_1} &= 2\alpha_0 D + 2\alpha_1 B - 2F = 0,
\end{align*}
which gives the $2\times 2$ system:
\begin{equation*}
    \begin{pmatrix} A & D \\ D & B \end{pmatrix}\begin{pmatrix} \alpha_0 \\ \alpha_1 \end{pmatrix} = \begin{pmatrix} E \\ F \end{pmatrix}.
\end{equation*}
When $AB - D^2 \neq 0$ (i.e., $g_{\mathrm{PG}}^b$ and $H$ are not perfectly collinear), the unique solution is:
\begin{equation*}
    \alpha_0^* = \frac{BE - DF}{AB - D^2}, \qquad \alpha_1^* = \frac{AF - DE}{AB - D^2}.
\end{equation*}

\paragraph{Independent batch trick.}
The quantities $A$, $B$, $D$ involve second moments of the estimators and can be consistently estimated from a single batch $\mathcal{B}$. However, $E = \E[{g_{\mathrm{PG}}^b}^\top \nabla_\theta J]$ and $F = \E[H^\top \nabla_\theta J]$ involve the unknown true gradient $\nabla_\theta J$.

To estimate $E$ and $F$ without knowing $\nabla_\theta J$, we draw an independent batch $\mathcal{B}'$ and compute an independent REINFORCE estimate $g_{\mathrm{PG}}' = (r' - b)S'$ from $\mathcal{B}'$. Since $g_{\mathrm{PG}}'$ is unbiased for $\nabla_\theta J$ and is independent of the batch-$\mathcal{B}$ quantities, the products
\begin{equation*}
    \hat{E} = {g_{\mathrm{PG}}^b}^\top g_{\mathrm{PG}}', \qquad \hat{F} = H^\top g_{\mathrm{PG}}'
\end{equation*}
are unbiased estimators of $E$ and $F$ respectively:
\begin{equation*}
    \E[\hat{E}] = \E_{\mathcal{B}}[g_{\mathrm{PG}}^b]^\top \E_{\mathcal{B}'}[g_{\mathrm{PG}}'] = (\nabla_\theta J)^\top (\nabla_\theta J) = \|\nabla_\theta J\|^2 = E,
\end{equation*}
where the factorization uses independence of $\mathcal{B}$ and $\mathcal{B}'$. The argument for $\hat{F}$ is analogous, using $\E_{\mathcal{B}'}[g_{\mathrm{PG}}'] = \nabla_\theta J$. This doubles the sample cost per step but yields a fully online, closed-form coefficient estimator with no hyperparameter tuning. The sample waste can be mitigated via \emph{cross-fitting}~\citep{fan2012refitted,chernozhukov2018dml}: swap the roles of $\mathcal{B}$ and $\mathcal{B}'$ to obtain a second estimate, then average, so that every trajectory contributes to a gradient estimate.
\end{proof}

\section{Extended Related Work}
\label{app:extended_related_work}

\paragraph{Stochastic computation graphs.}
The stochastic computation graph framework~\citep{schulman2015scg} and its extension DiCE~\citep{foerster2018dice} are discussed in \Cref{sec:related_work}. We add one technical note: DiCE targets correct higher-order gradient estimators by automatically tracking which nodes are ``on the graph.'' HR's distinction between frozen and reparameterized history relates to this bookkeeping, though HR targets first-order variance reduction rather than higher-order derivatives.

\paragraph{Reparameterization gradients.}
The reparameterization trick~\citep{kingma2014,rezende2014stochastic} has been generalized beyond location-scale families to distributions amenable to approximate standardization~\citep{ruiz2016}, implicit reparameterization via the implicit function theorem~\citep{figurnov2018}, acceptance-rejection schemes~\citep{naesseth2016}, and optimal-transport constructions~\citep{jankowiak2018}. For discrete distributions, Gumbel-Softmax~\citep{jang2017} and Concrete~\citep{maddison2017} relaxations provide biased but low-variance pathwise gradients. Our work applies reparameterization to the \emph{conditioning history} within a policy gradient, rather than to the sampling distribution itself.

\paragraph{Discrete gradient estimation.}
The high variance of REINFORCE~\citep{williams1992} for discrete variables has motivated multi-sample and leave-one-out baselines~\citep{kool2019buy}, score-function/relaxation hybrids such as REBAR~\citep{tucker2017} and its learned-surrogate extension RELAX~\citep{grathwohl2018}, the straight-through estimator~\citep{bengio2013}, antithetic and Rao--Blackwellized methods~\citep{yin2019arm,dong2020disarm,yin2019arsm}, and foundational stochastic backpropagation techniques~\citep{rezende2014stochastic,mnih2014nvil}. These methods primarily target single-step or parallel discrete choices; while extensions to sequential settings exist~\citep{paulus2021rao}, they do not exploit autoregressive conditioning structure---the specific gap that HR addresses.

\paragraph{Credit assignment in sequences.}
GRPO, PPO, and process reward models are discussed in \Cref{sec:related_work}. In continuous-action RL, Q-Prop~\citep{gu2017qprop} uses a reparameterized critic gradient as a control variate for policy gradients, but relies on an off-policy critic and environment-provided states rather than the autoregressive conditioning structure. HR leverages the specific structural assumption that in autoregressive models, the ``environment'' is a concatenation of past actions, making the conditioning context differentiable with respect to the policy parameters.

\paragraph{Curriculum learning.}
Curriculum methods can ease sparse-reward learning by simplifying tasks or narrowing the search space~\citep{florensa2017reverse,salimans2018montezuma,xi2024r3,amani2026}; HR instead modifies the gradient estimator.

\paragraph{Model-based RL and world models.}
Dreamer~\citep{hafner2020dreamer} and DreamerV3~\citep{hafner2023dreamerv3} learn a differentiable world model and backpropagate through imagined trajectories to compute policy gradients. HR is structurally analogous: the policy serves as its own dynamics model, with no auxiliary model required (\Cref{rem:self_model}). The key difference is that Dreamer operates in the standard MDP setting and requires learning a separate world model, while HR exploits the fact that in autoregressive generation the policy \emph{already is} the dynamics model. DreamerV3 uses straight-through gradients for discrete actions in the learned world model---closely related to our ST-Gumbel baseline---while HR's embedding-level relaxation provides a richer differentiable path.

\paragraph{Positioning of HR.}
HR occupies a distinct position in this landscape. Unlike REBAR/RELAX, which construct control variates from differentiable relaxations of the sampling process, HR reparameterizes the \emph{conditioning history} and does not require a differentiable reward. Unlike full pathwise estimation (\Cref{sec:bg_reparam}), HR retains the score function for the current token and does not differentiate through $r$. Unlike world-model methods, HR requires no auxiliary model---it reparameterizes through the policy's own autoregressive structure. The OHR variant provides a closed-form mechanism for controlling the bias--variance tradeoff, in contrast to RELAX's learned control variate.

\section{Complete Gaussian Model Formulas}
\label{app:gaussian_formulas}

We collect all closed-form expressions for the Gaussian linear Markov model from \Cref{sec:theory}.

\subsection{General Setup}

Policy: $p_W(x_t \mid x_{t-1}) = \mathcal{N}(Wx_{t-1}, I)$, with $x_t = Wx_{t-1} + \varepsilon_t$, $\varepsilon_t \overset{\mathrm{iid}}{\sim} \mathcal{N}(0, I)$.

Unrolled trajectory: $x_t = W^{t+1}x_{-1} + \sum_{k=0}^{t}W^{t-k}\varepsilon_k$.

Score: $\partial_W \log p_W(x_t \mid x_{t-1}) = \varepsilon_t x_{t-1}^\top$.

History-pathwise correction: for the linear model, $\frac{\partial \log p_W(x_t \mid x_{t-1})}{\partial x_j} = (W^\top)^{t-j} \varepsilon_t$ for $j < t$, obtained by applying the chain rule through the reparameterized dynamics $\chi_t = W\chi_{t-1} + \varepsilon_t$.

\subsection{Linear Reward, General $N$}

\begin{align*}
    r &= a^\top x_N, \qquad J = a^\top W^{N+1}x_{-1}, \\
    \nabla_W J &= \sum_{k=0}^{N}(W^\top)^k a\;(W^{N-k}x_{-1})^\top, \\
    g_{\mathrm{PG}} &= (a^\top x_N)\sum_{t=0}^{N}\varepsilon_t\,x_{t-1}^\top, \\
    C &= \sum_{t=1}^{N}\sum_{k=0}^{t-1}(W^\top)^{t-k}\varepsilon_t\,x_{k-1}^\top, \\
    \mathrm{Bias}_{\mathrm{HR}} &= \E\bigl[(a^\top x_N - J)\,C\bigr] = \sum_{t=1}^{N}\sum_{k=0}^{t-1}(W^\top)^{t-k}\,\E[\varepsilon_t\,(a^\top x_N)\, x_{k-1}^\top].
\end{align*}

The bias computation requires expanding $a^\top x_N = a^\top W^{N+1}x_{-1} + \sum_{\ell=0}^{N}a^\top W^{N-\ell}\varepsilon_\ell$ and using $\E[\varepsilon_t \varepsilon_\ell^\top] = I\,\delta_{t\ell}$ and $\E[\varepsilon_t] = 0$. For $N=1$, this reduces to $\mathrm{Bias}_{\mathrm{HR}} = W^\top a\, x_{-1}^\top$.

\subsection{$L^2$ Reward, $N=1$}

\begin{align*}
    r &= -\tfrac{1}{2}\norm{x_1 - y}^2, \qquad \delta = W^2x_{-1} - y, \\
    J &= -\tfrac{1}{2}\bigl(\norm{\delta}^2 + \Tr(WW^\top) + d\bigr), \\
    \nabla_W J &= -\delta(Wx_{-1})^\top - W^\top\delta\,x_{-1}^\top - W, \\
    \mathrm{Bias}_{\mathrm{HR}} &= -W^\top\delta\,x_{-1}^\top.
\end{align*}

The $-W$ term in $\nabla_W J$ arises from the variance of the noise: $\E_{\varepsilon_1}[\varepsilon_1\varepsilon_1^\top] = I$ contributes a gradient through the quadratic reward. This term is captured by neither the score-function estimator nor the history-pathwise correction alone, but is recovered by their combination in the OHR framework.

\subsection{Variance Reduction and Credit Assignment}
\label{app:theory_variance}

We validate the theoretical predictions by computing the trace of variance, MSE, and SNR for each estimator using $M = 50{,}000$ Monte Carlo trajectories, comparing against the closed-form gradient.

\paragraph{OHR achieves growing variance reduction with horizon.}
Fixing $W$ with spectral radius $\rho(W) = 0.95$ and sweeping $N \in \{1, \ldots, 20\}$, the variance reduction ratio $\Tr(\Var(g_{\mathrm{PG}}^b)) / \Tr(\Var(g_{\mathrm{OHR}}))$ increases monotonically with $N$, reaching $1.08\times$ for the linear reward and $1.18\times$ for the $L^2$ reward at $N = 20$ (\Cref{fig:gaussian_diagnostics}).
This confirms that history reparameterization is most beneficial precisely when credit assignment is hardest. 
The OHR estimator also achieves higher SNR than baselined PG at all horizons tested, and the optimal coefficient $\alpha_1^*$ stabilizes at moderate horizons.

\paragraph{Scalar closed form.}
In the scalar case ($d = 1$), the limiting variance reduction admits the closed form $\rho^2 \to w^2/(1 + w^2)$ as $N \to \infty$ for $|w| < 1$, which is strictly positive for any $w \neq 0$ and increases with $|w|$.

\begin{figure}[t]
    \centering
    \includegraphics[width=\textwidth]{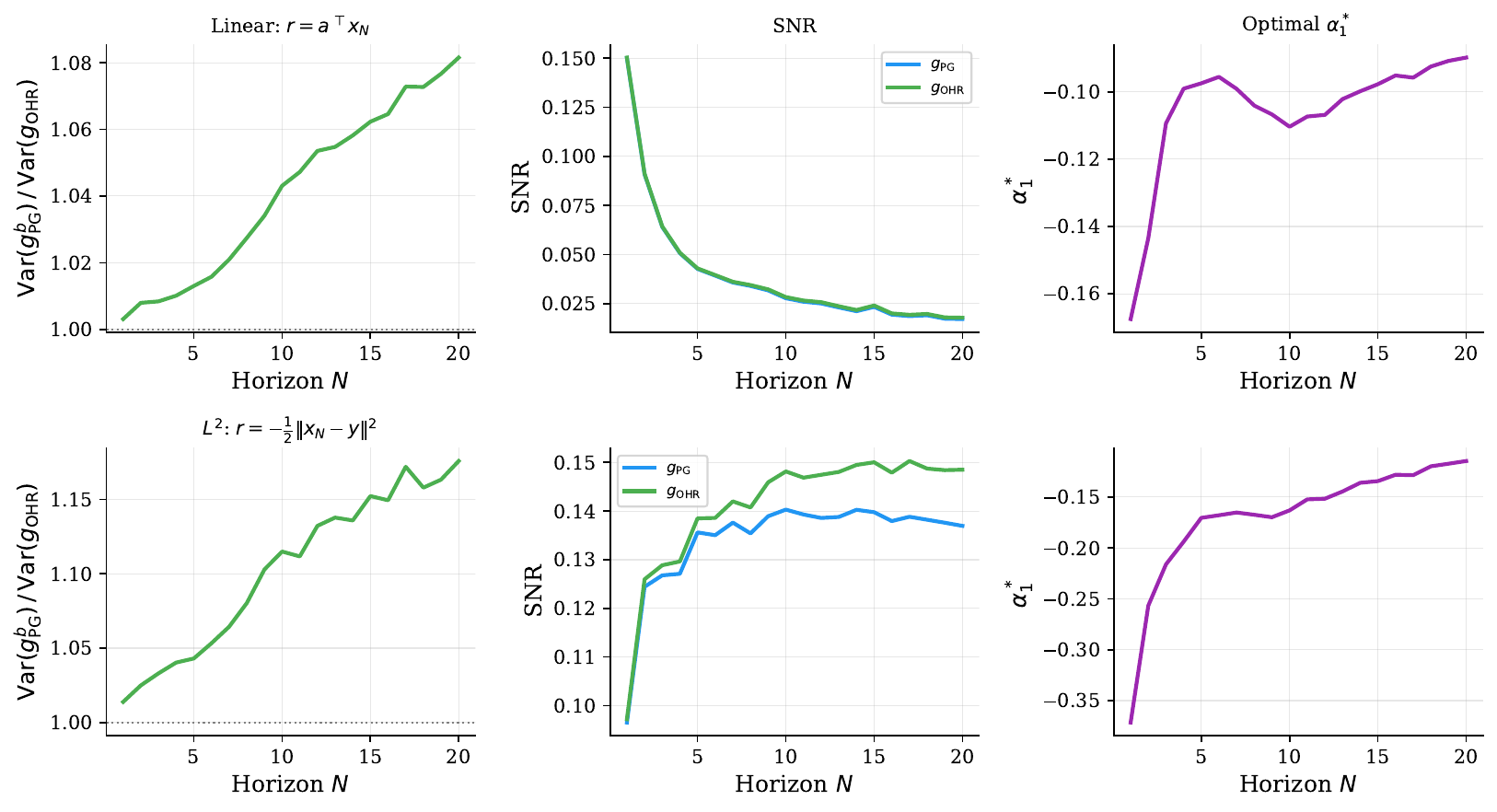}
    \caption{Gaussian linear Markov model ($d=5$, $\rho(W) = 0.95$). Top: linear reward. Bottom: $L^2$ reward. Left: variance reduction ratio (OHR over baselined PG), growing monotonically with $N$ to $1.08\times$ (linear) and $1.18\times$ ($L^2$). Center: SNR of both estimators. Right: optimal coefficient $\alpha_1^*$.}
    \label{fig:gaussian_diagnostics}
\end{figure}

\section{Supplementary Experiments}
\label{app:supp_experiments}

This section collects supplementary figures and analysis for the experiments in \Cref{sec:experiments}, organized by task.

\subsection{Embedding-Based Reparameterization: Training Dynamics}
\label{app:supp_embedding}

To test whether the oracle variance reduction from \Cref{sec:exp_embedding} translates to faster training, \Cref{fig:training_embedding} reports reward and per-step trace-variance for PG and OHR across $N \in \{5, 15, 25\}$ (using $V = 32$, $d_e = 8$, with reward for matching a fixed final target token; see \Cref{app:experiment_details} for full details). OHR converges faster and reaches a higher final reward at every horizon, with the gap widening as $N$ grows: at short horizons it acts mainly as a speed-up over PG, at intermediate horizons it becomes the difference between solving and stalling, and at long horizons both estimators struggle but OHR retains a consistent lead. OHR's per-sample variance is consistently lower than PG's at $N=15$ and $N=25$. The variance spikes at $N=5$ are a saturation artifact: once the task is solved, $r-b\!\approx\!0$ across the batch and both estimators' variances collapse to the floating-point floor.

\begin{figure}[t]
    \centering
    \includegraphics[width=\textwidth]{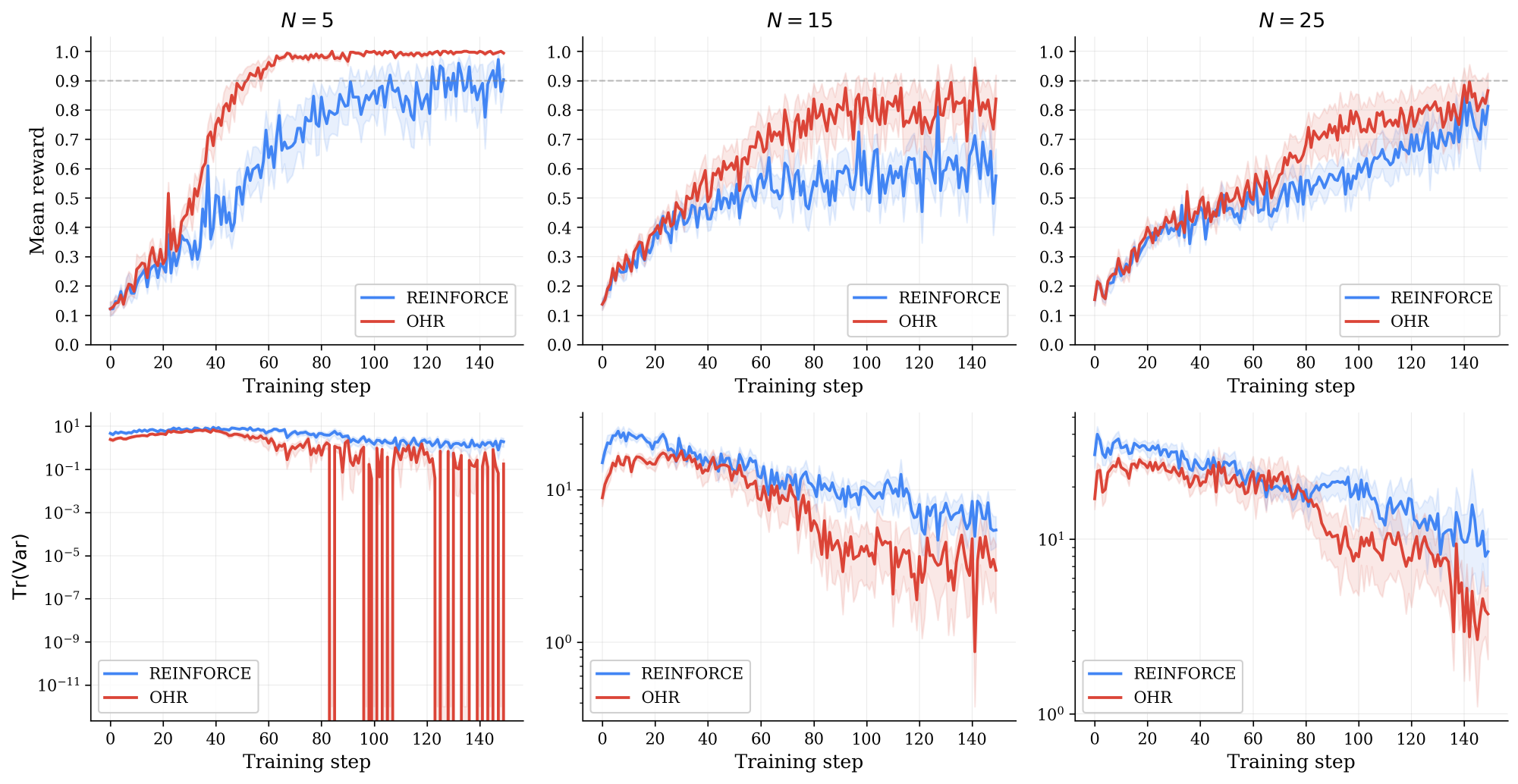}
    \caption{Continuous Embedding training: PG (blue) vs.\ OHR (red), 10 seeds per $N$, $\tau{=}3.0$. Top: reward curves with mean $\pm$ s.e. Bottom: estimator variance during training (log scale). OHR consistently achieves better training dynamics and lower variance.}
    \label{fig:training_embedding}
\end{figure}

\subsection{T-Maze: Oracle Variance Reduction}
\label{app:supp_tmaze}

\Cref{fig:tmaze_vr} shows oracle variance reduction for the T-maze measured on frozen pretrained policies ($M{=}1{,}024$ per-sample gradients, 10 seeds). PG variance grows steeply with corridor length while OHR variance grows much slower, yielding a ratio that increases monotonically with~$N$. The optimal coefficient $\alpha^* \approx -2$ and cosine similarity $\cos(g_{\mathrm{PG}}, C) \approx 0.3$ are stable across corridor lengths.

\begin{figure}[t]
    \centering
    \includegraphics[width=\textwidth]{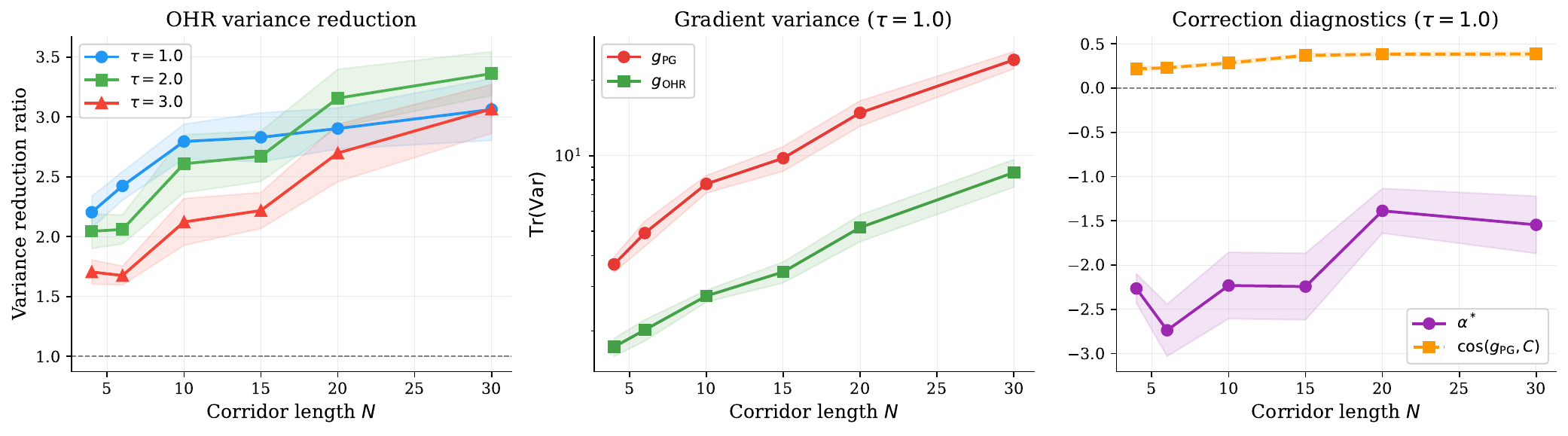}
    \caption{T-maze with GRU policy and embedding-based action feedback (10 seeds). Left: variance reduction ratio grows monotonically with corridor length $N$ across all $\tau$. Center: PG variance (blue) grows steeply with $N$ while OHR variance (green) grows much slower (log scale). Right: $\alpha^* \approx -2$ and $\cos(g_{\mathrm{PG}}, C) \approx 0.3$ are stable across $N$.}
    \label{fig:tmaze_vr}
\end{figure}

\subsection{TSP: Training Curves}
\label{app:supp_tsp}

\Cref{fig:tsp_learning} compares REINFORCE and OHR training across $n \in \{8, 10, 15, 20\}$ cities. Both methods converge to similar final tour lengths, but OHR achieves $1.15$--$1.25\times$ lower gradient variance throughout training.

\begin{figure}[t]
    \centering
    \includegraphics[width=\textwidth]{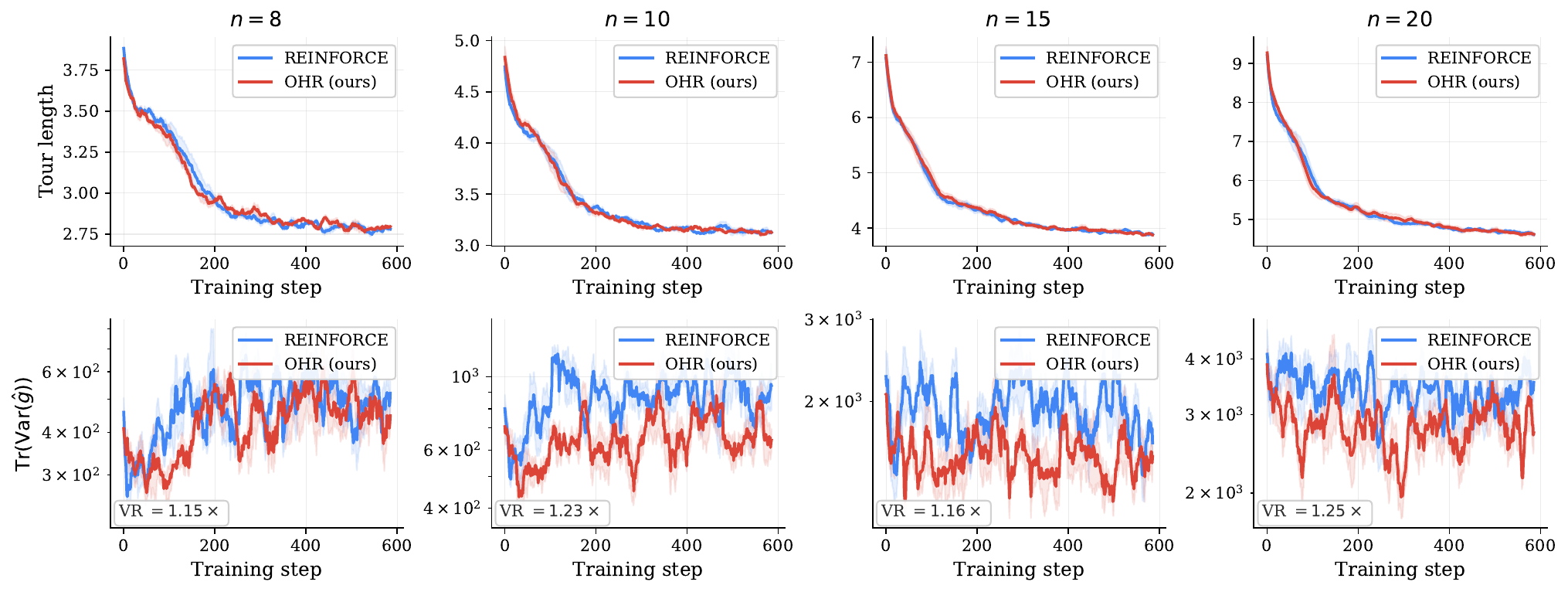}
    \caption{TSP training: REINFORCE (blue) vs.\ OHR (red), $M{=}16$, $\tau{=}1.0$, 4 seeds. Top: tour length converges to similar final values. Bottom: OHR gradient variance (red) is consistently lower than REINFORCE (blue), with VR annotations showing $1.15$--$1.25\times$ reduction.}
    \label{fig:tsp_learning}
\end{figure}

\subsection{Composability with Reward-Side Methods}
\label{app:supp_composability}

HR operates on the score function $\nabla_\theta \log p_\theta$; existing variance reduction techniques---greedy rollout baselines~\citep{kool2019attention}, POMO~\citep{kwon2020pomo}, leave-one-out (RLOO) baselines~\citep{kool2019buy}, and learned critics---operate on the advantage $(r - b)$. Since these two sources of variance reduction are structurally orthogonal, they should compose. We test this prediction on both the T-maze and TSP.

\paragraph{T-maze: HR + learned critic.}
We add a small critic network (MLP on the GRU hidden state at the junction, trained with MSE loss) and compare four conditions: PG and OHR, each with batch-mean or learned-critic baseline (\Cref{fig:tmaze_compose}). At $N{=}10$, OHR + mean baseline solves the task first (reaching reward $1.0$ around step 150), outpacing both REINFORCE variants (${\sim}$step 200). With the learned critic, OHR reduces variance at all tested corridor lengths, reaching $1.70\times$ at $N{=}15$ (\Cref{tab:composability}). With the batch-mean baseline, OHR improves variance at $N{=}6$ and $N{=}15$ but not at $N{=}10$ ($0.86\times$), indicating that the interaction between HR and a given baseline can be problem-dependent.

\paragraph{TSP: HR + RLOO baseline.}
We combine OHR with the leave-one-out baseline of~\citet{kool2019buy}, sampling $k{=}8$ tours per instance (\Cref{fig:tsp_compose}). The variance reductions compose near-multiplicatively: at $n{=}10$, RLOO alone gives $2.45\times$, OHR alone gives $1.27\times$, and OHR + RLOO gives $3.55\times$ (vs.\ a product prediction of $3.11\times$). The combined estimator achieves the lowest gradient variance across all conditions at every problem size tested.

\begin{figure}[t]
    \centering
    \includegraphics[width=\textwidth]{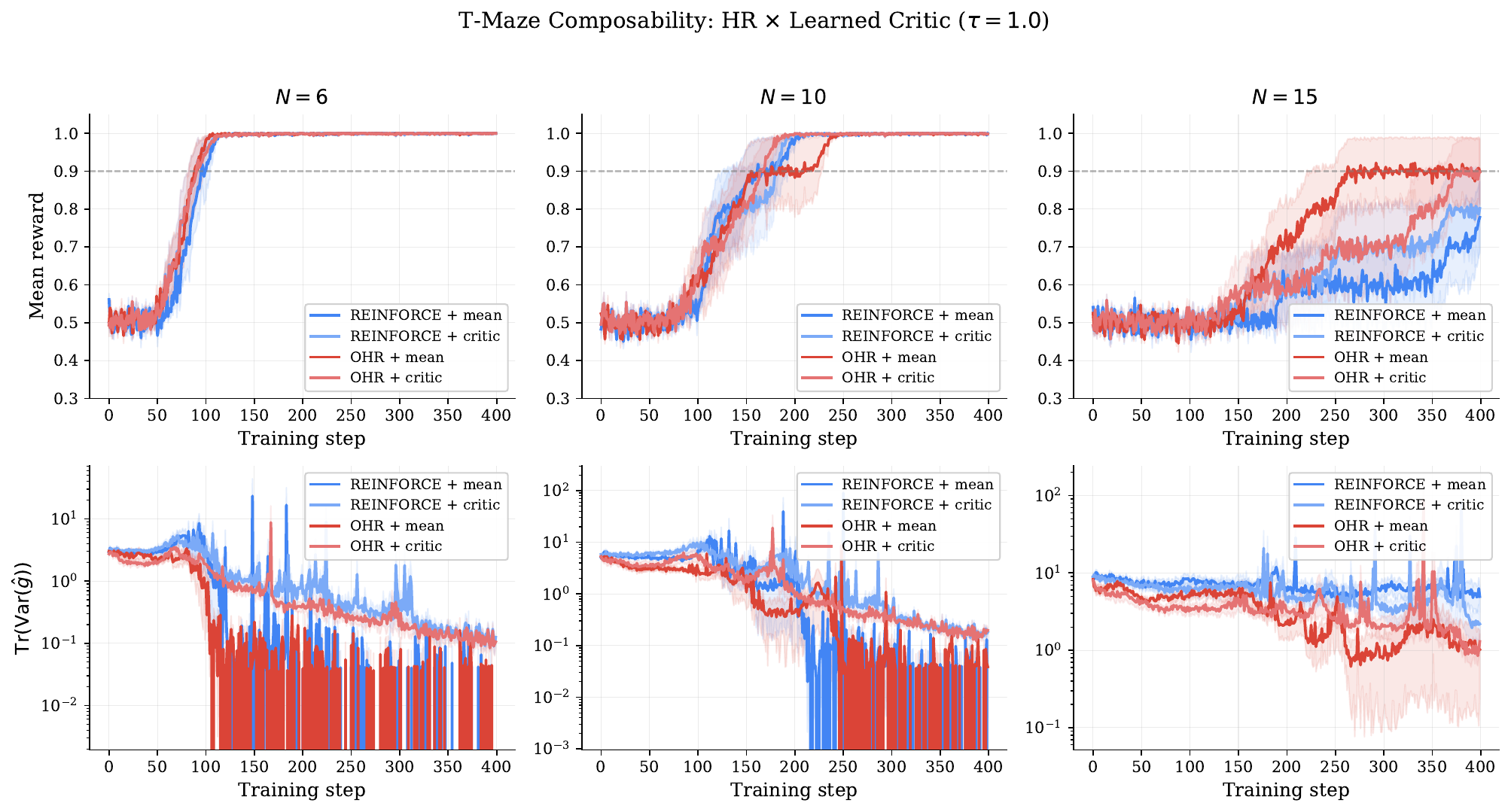}
    \caption{T-maze composability ($\tau{=}1.0$, 10 seeds). At $N{=}10$, OHR (red) solves the task faster than REINFORCE (blue) regardless of baseline. Solid: mean baseline; dashed: learned critic.}
    \label{fig:tmaze_compose}
\end{figure}

\begin{figure}[t]
    \centering
    \includegraphics[width=\textwidth]{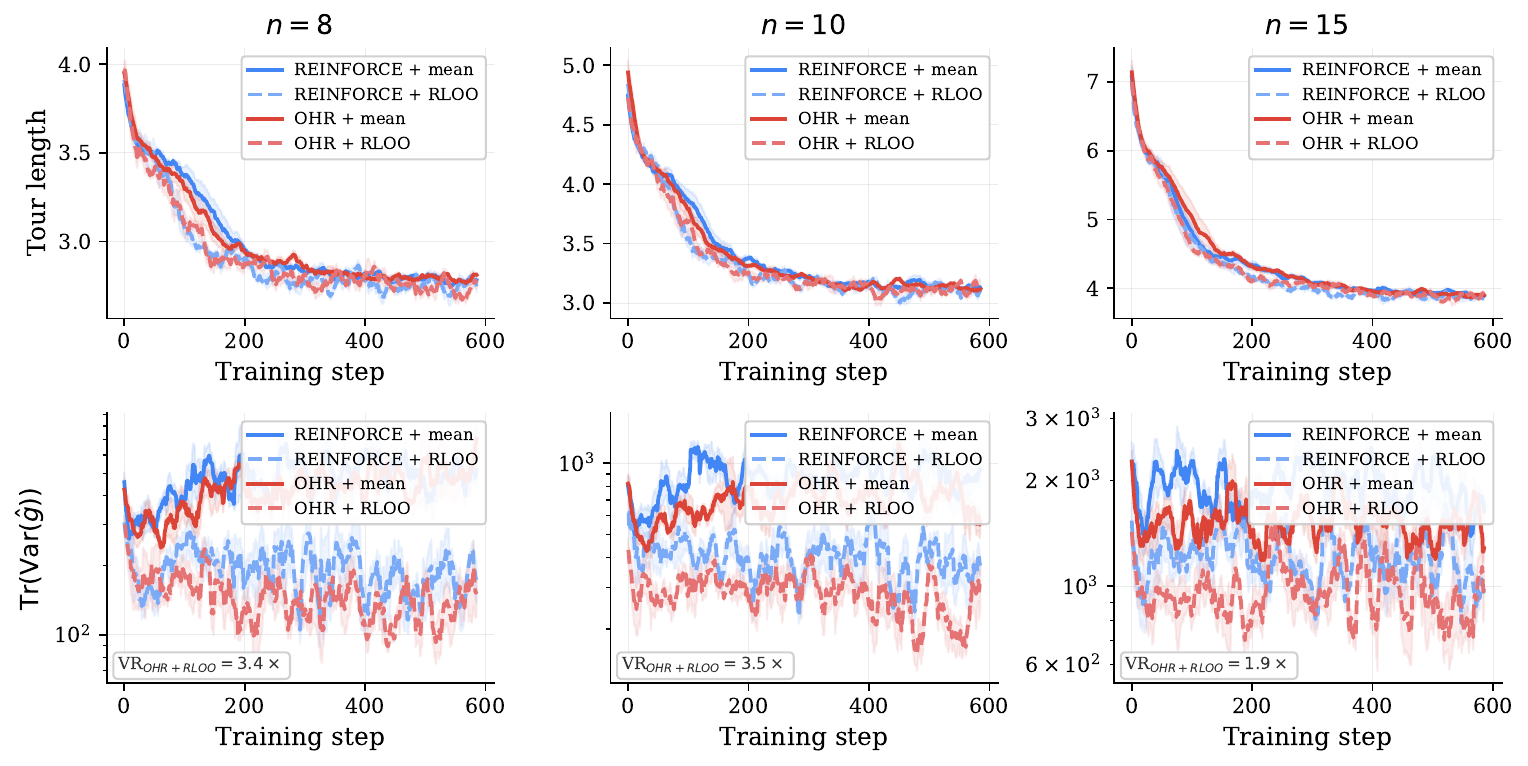}
    \caption{TSP composability ($\tau{=}1.0$, $M{=}16$, 4 seeds). OHR + RLOO (dashed red) achieves the lowest variance at all problem sizes, with combined VR of $3.4$--$3.5\times$ at $n{=}8$--$10$.}
    \label{fig:tsp_compose}
\end{figure}

\begin{table}[t]
    \centering
    \caption{Composability of HR with reward-side methods. Left: T-maze VR with and without learned critic. Right: TSP VR decomposition showing near-multiplicative composition of OHR and RLOO.}
    \label{tab:composability}
    \small
    \begin{minipage}[t]{0.45\textwidth}
    \centering
    \textbf{T-maze: PG/OHR variance ratio}\\[4pt]
    \begin{tabular}{lcc}
        \toprule
        $N$ & Mean baseline & Learned critic \\
        \midrule
        6  & $1.13\times$ & $1.54\times$ \\
        10 & $0.86\times$ & $1.08\times$ \\
        15 & $4.36\times$ & $1.70\times$ \\
        \bottomrule
    \end{tabular}
    \end{minipage}
    \hfill
    \begin{minipage}[t]{0.50\textwidth}
    \centering
    \textbf{TSP: VR vs.\ PG + mean baseline}\\[4pt]
    \begin{tabular}{lccc}
        \toprule
        $n$ & RLOO & OHR & OHR+RLOO \\
        \midrule
        8  & $3.08\times$ & $0.94\times$ & $3.44\times$ \\
        10 & $2.45\times$ & $1.27\times$ & $3.55\times$ \\
        15 & $1.73\times$ & $1.12\times$ & $1.89\times$ \\
        \bottomrule
    \end{tabular}
    \end{minipage}
\end{table}

\section{Experiment Details}
\label{app:experiment_details}

\subsection{Experiment Setup}
\subsubsection{Gaussian Linear Markov Model (\Cref{sec:theory})}

\begin{itemize}
    \item Dimension: $d = 5$.
    \item Weight matrix (N-sweep): $W = 0.95 \cdot Q$ where $Q$ is a random orthogonal matrix ($\rho(W) = 0.95$).
    \item Horizons: $N \in \{1, 2, \ldots, 20\}$.
    \item Reward: linear $r = a^\top x_N$ with $a = (1, -1, 1, -1, 1)$, and $L^2$ reward $r = -\frac{1}{2}\|x_N - y\|^2$.
    \item Initial state: $x_{-1} = \mathbf{1}/\sqrt{d}$.
    \item Number of Monte Carlo samples: $M = 50{,}000$.
    \item Baseline: $b = \bar{r}_{\mathcal{B}}$ (batch mean).
    \item Exact gradient computed from closed-form formulas in \Cref{app:gaussian_formulas}.
\end{itemize}

\subsubsection{Embedding-Based HR (\Cref{sec:exp_embedding})}

\begin{itemize}
    \item Vocabulary: $V = 8$, embedding dimension $d_e = 16$.
    \item Architecture: embedding layer + linear output head (no hidden layers).
    \item Reward: $r = 1$ if generated sequence matches fixed target, $r = 0$ otherwise.
    \item Temperature: $\tau \in \{1.0, 2.0, 3.0\}$.
    \item Horizons: $N \in \{3, 5, 7, 9, 11, 13, 15, 18, 20\}$.
    \item Per-sample gradients: $M = 2{,}000$, averaged over 5 seeds.
    \item Baseline: $b = \bar{r}_{\mathcal{B}}$ (batch mean).
\end{itemize}

\subsubsection{Continuous Embedding-Based HR Training Dynamics (\Cref{sec:exp_embedding})}
\begin{itemize}
    \item Vocabulary: $V = 32$, embedding dimension $d_e = 8$.
    \item Architecture: embedding layer + linear output head (no hidden layers).
    \item Reward: $r = 1$ if the final sampled token matches the fixed target token $x^* = V/2$, $r = 0$ otherwise.
    \item Initial token: each trajectory starts from a token sampled uniformly from $\{0, \ldots, V-1\}$, re-sampled per training step to prevent path memorization.
    \item Temperature: $\tau = 3.0$.
    \item Horizons: $N \in \{5, 15, 25\}$.
    \item Pretraining: supervised on a structured transition kernel for $200$ steps, $\text{lr}=5\!\times\!10^{-3}$, batch size $128$.
    \item Optimizer: Adam, $\text{lr}=10^{-2}$, batch size $M = 32$, $150$ training steps.
    \item OHR centering: EMA across training steps with decay $\rho = 0.99$ (in-batch centering would collapse OHR to PG).
    \item Coefficient clipping: $\hat{\alpha}_{1,k}$ clipped to $[-2, 2]$; clipping is used only in this training experiment.
    \item Seeds: $10$ for mean $\pm$ s.e.\ bands.
    \item Baseline: $b = \bar{r}_{\mathcal{B}}$ (batch mean).
\end{itemize}

\subsubsection{T-Maze (\Cref{sec:exp_tmaze})}

\begin{itemize}
    \item Policy: GRU with hidden size $h = 16$, 2-way categorical output.
    \item Action embedding: $E_a \in \R^{2 \times h}$, learned during pretraining.
    \item Corridor lengths: $N \in \{4, 6, 10, 15, 20, 30\}$.
    \item Temperature: $\tau \in \{1.0, 2.0, 3.0\}$.
    \item Per-sample gradients: $M = 1{,}024$, averaged over 10 seeds.
    \item Pretraining: supervised cross-entropy at the junction, checkpoint selected at ${\sim}0.78$ reward.
    \item Baseline: $b = \bar{r}_{\mathcal{B}}$ (batch mean).
    \item Oracle measurement: variance computed on frozen pretrained policies (no gradient updates during measurement).
\end{itemize}

\subsubsection{T-Maze Training (\Cref{sec:exp_tmaze}, Learning Curves)}

\begin{itemize}
    \item Optimizer: SGD with constant learning rate.
    \item Learning rate: tuned per estimator via grid search over $\{10^{-4}, 10^{-3}, 10^{-2}, 10^{-1}\}$.
    \item Corridor lengths: $N \in \{10, 15, 20\}$.
    \item Temperature: $\tau = 1.0$.
    \item Batch size: $M = 128$.
    \item Training steps: 500.
    \item Seeds: 10 per corridor length per estimator.
    \item Initialization: pretrained checkpoint at ${\sim}0.78$ reward (same as oracle experiments).
\end{itemize}

\subsubsection{TSP Training (\Cref{sec:exp_tsp})}

\begin{itemize}
    \item Optimizer: Adam with learning rate $10^{-4}$.
    \item Problem sizes: $n \in \{8, 10, 15, 20\}$.
    \item Temperature: $\tau = 1.0$.
    \item Batch size: $M = 16$ (one tour per instance).
    \item Training steps: 600.
    \item Seeds: 4 per problem size per estimator.
    \item Architecture: embedding dimension $d = 64$, $L = 2$ encoder layers, 4 attention heads, single-head pointer decoder (${\sim}120$K parameters).
\end{itemize}

\subsubsection{Composability (\Cref{sec:exp_composability})}

\begin{itemize}
    \item T-maze critic: 2-layer MLP on GRU hidden state at junction, trained with MSE loss, learning rate $10^{-3}$.
    \item TSP RLOO: $k = 8$ tours per instance, leave-one-out baseline per~\citet{kool2019buy}.
    \item All other hyperparameters match the corresponding base experiments above.
\end{itemize}

\end{document}